\documentclass{article} % For LaTeX2e

\usepackage[table]{xcolor}  
\usepackage{iclr2026_conference,times}

\iclrfinalcopy

\usepackage{amsmath,amsfonts,bm}

\def\eqref#1{equation~\ref{#1}}
\def\1{\bm{1}}

\def\mF{{\bm{F}}}

\def\mP{{\bm{P}}}

\DeclareMathAlphabet{\mathsfit}{\encodingdefault}{\sfdefault}{m}{sl}
\SetMathAlphabet{\mathsfit}{bold}{\encodingdefault}{\sfdefault}{bx}{n}

\newcommand{\E}{\mathbb{E}}

\newcommand{\R}{\mathbb{R}}

\DeclareMathOperator*{\argmax}{arg\,max}
\DeclareMathOperator*{\argmin}{arg\,min}

\usepackage{hyperref}
\usepackage{url}

\usepackage{amsmath}
\usepackage{amssymb}
\usepackage{mathtools}
\usepackage{amsthm}

\usepackage{microtype}
\usepackage{graphicx}
\usepackage{booktabs} % for professional tables

\usepackage{algorithm}
\usepackage{algpseudocode}
\usepackage{graphicx} 
\usepackage{todonotes}

\usepackage{xcolor}
\usepackage{dsfont}
\usepackage{nicefrac}
\usepackage{tcolorbox}
\usepackage{tabularx}
\usepackage{pifont}
\usepackage[inline]{enumitem}
\usepackage{multirow}
\usepackage{caption}
\usepackage{wrapfig}
\usepackage[export]{adjustbox}
\usepackage[list=true]{subcaption}
\usepackage{colortbl}
\definecolor{lightgray}{RGB}{242, 242, 242}
\usepackage{float}
\usepackage[capitalize,noabbrev]{cleveref}
\crefname{assumption}{Assumption}{Assumptions}
\crefname{algorithm}{Algorithm}{Algorithms}

\usepackage{longtable}
\usepackage[title]{appendix}

\usepackage{tcolorbox}
\usepackage{xspace}
\DeclareRobustCommand{\ie}{i.e.,\@\xspace}
  
\DeclareRobustCommand{\eg}{e.g.,\@\xspace}
\DeclareRobustCommand{\wrt}{w.r.t.\@\xspace}
\usepackage[capitalize]{cleveref}
\usepackage{tikz}
\usepackage{tikz-cd}
\usetikzlibrary{shapes,backgrounds}
\usepackage{array}
\usepackage{enumitem}

\usepackage{wrapfig}
\usepackage[export]{adjustbox}
\usepackage[list=true]{subcaption}
\usepackage[normalem]{ulem}
\usepackage{rotating}   % for \rotatebox (rotates text)
\usepackage{cleveref}
\usepackage{adjustbox}
\usepackage{twoopt}

\usepackage[T1]{fontenc}
\usepackage[utf8]{inputenc}
\usepackage{amsmath,amssymb}
\usepackage{booktabs}      
\usepackage{multirow}        

\usepackage{acronym}		% for acronyms
\usepackage{amsthm}
\usepackage{thmtools, thm-restate}
\declaretheorem[numberwithin=section]{thm}
\declaretheorem[sibling=thm]{theorem}
\declaretheorem[sibling=thm]{lemma}

\declaretheorem[numberwithin=section]{assumption}
\declaretheorem[]{proof sketch}
\declaretheorem[]{definition}
\declaretheorem[]{proposition}

\usepackage{amsfonts}[mathscr]
\usepackage{amssymb}
\usepackage{amsmath}
\DeclareMathOperator*{\EV}{\mathbb{E}}

\newcommand{\A}{\mathcal{A}}
\newcommand{\Lfunc}{\mathcal{L}}

\newcommand{\X}{\mathcal{X}}

\newcommand{\D}{\mathcal{D}}

\newcommand{\Q}{\mathcal{Q}}
\newcommand{\G}{\mathcal{G}}
\newcommand{\mypar}[1]{\textbf{#1.}}

\newcommand{\entropy}{\mathcal{H}}
\newcommand{\der}{\mathrm{d}}

\newcommand{\noise}{U}
\newcommand{\bias}{b}
\newcommand{\hist}{\mathcal{T}}

\newcommand{\step}{\gamma}

\usepackage{pifont}
\usepackage{makecell}
\definecolor{darkgreen}{RGB}{0,100,0}
\definecolor{darkorange}{RGB}{255,140,0}

\newcommand{\AlgNameLong}{Flow-Expander \xspace}
\newcommand{\AlgNameShort}{\textsc{\small{FE}}\xspace}
\newcommand{\AlgNameDef}{\textbf{F}low \textbf{E}xpander  (\textsc{\small{FE}}\xspace)}

\newcommand{\AlgNameDefLocal}{\textbf{L}ocal \textbf{FE}    (\textsc{\small{L-FE}}\xspace)}
\newcommand{\AlgNameShortLocal}{\textsc{\small{L-FE}}\xspace}

\newcommand{\AlgNameDefGlobal}{\textbf{G}lobal \textbf{FE}    (\textsc{\small{G-FE}}\xspace)}
\newcommand{\AlgNameShortGlobal}{\textsc{\small{G-FE}}\xspace}

\newcommand{\AlgNameDefExplore}{\textbf{N}oised \textbf{S}pace \textbf{E}xploration    (\textsc{\small{NSE}}\xspace)}
\newcommand{\AlgNameShortExplore}{\textsc{\small{NSE}}\xspace}

\newcommand{\AlgNameShortFDC}{\textsc{\small{FDC}}\xspace}
\newcommand{\AlgNameShortSMEME}{\textsc{\small{S-MEME}}\xspace}
\newcommand{\AlgNameShortCONSTR}{\textsc{\small{CONSTR}}\xspace}

\newcommand{\FineTuningSolver}{\textsc{\small{FineTuningSolver}}\xspace}

\newcommand{\ExpandThenProject}{\textsc{\small{ExpandThenProject}}\xspace}

\definecolor{myviolet}{rgb}{0.6, 0.4, 0.8}
\definecolor{mygreen}{rgb}{0.0, 0.5, 0.0}
\definecolor{myorange}{rgb}{1., 0.8, 0.17}

\newcommand{\debug}[1]{#1}

\newcommand{\newmacro}[2]{\newcommand{#1}{{#2}}}		% for shorthand definitions
\newcommand{\dual}{h}
\newcommand{\run}{k}

\newcommand{\state}{\dual}
\newcommand{\curr}[1][\state]{\debug{#1}^{\run}}		% for current value (X by default)
\renewcommand{\next}[1][\state]{\debug{#1}^{\run+1}}		% for current value (X by default)

\newcommand{\efftime}{\tau}
\newcommand{\apt}[2][]{\state^{#1}(#2)}	
\newcommand{\ctime}{t}

\newmacro{\temp}{\eta}		% for learning rate
\newmacro{\points}{\mathcal{Z}}		% for point set
\newmacro{\intpoints}{\points^{\circ}}		%for point set interior
\newmacro{\point}{\dual}		% for generic point
\newmacro{\pointalt}{\alt\point}		% for alternate point
\newmacro{\ctimealt}{s}		% for dummy continuous time
\newmacro{\cstart}{0}		% for continuous time start

\newmacro{\horizon}{T}		% for horizon

\newmacro{\vecfield}{V}		% for vector field

\newmacro{\signal}{V}		% for signal
\newmacro{\error}{W}		% for error
\newmacro{\brown}{W}		% for Wiener process

\newmacro{\dstate}{Y}		% for other iterate

\newmacro{\flowmap}{\Theta}		% for (semi)flows
\newcommandtwoopt{\flow}[2][\ctime][\point]{\flowmap_{#1}(#2)}
\newmacro{\minmax}{\Phi}		% for minmax objective

\newmacro{\minvar}{x}		% for minimization variable
\newmacro{\minvaralt}{\alt x}		% for alternate minvar
\newmacro{\minvars}{\mathcal{X}}		% for minvar space

\newmacro{\maxvar}{y}		% for maximization variable
\newmacro{\maxvaralt}{\alt y}		% for alternate maxvar
\newmacro{\maxvars}{\mathcal{Y}}		% for maxvar space

\newmacro{\minsol}{\sol[\minvar]}		% for minimization solution
\newmacro{\maxsol}{\sol[\maxvar]}		% for maximization solution

\newcommand{\sol}[1][\point]{#1^{\ast}}		% for solutions (x by default)
\newmacro{\set}{\mathcal{S}}		% for generic set

\newmacro{\open}{\mathcal{U}}		% for open sets
\newmacro{\closed}{\mathcal{C}}		% for closed sets
\newmacro{\cpt}{\mathcal{K}}		% for compact sets
\newmacro{\nhd}{\mathcal{U}}		% for neighborhoods

\newacro{APT}{asymptotic pseudotrajectory}
\newacroplural{APT}[APTs]{asymptotic pseudotrajectories}
\newacro{GD}{gradient dynamics}
\newacro{GF}{gradient flow}
\newacro{ICT}{internally chain-transitive}
\newacro{MDS}{martingale difference sequence}
\newacro{NE}{Nash equilibrium}
\newacroplural{NE}[NE]{Nash equilibria}
\newacro{ODE}{ordinary differential equation}
\newacro{SA}{stochastic approximation}
\newacro{SFO}{stochastic first-order oracle}
\newacro{SG}{stochastic gradient}
\newacro{SP}{saddle-point}
\newacro{WAC}{weak asymptotic coercivity}

\newacro{AH}{Arrow\textendash Hurwicz}
\newacro{BDG}{Burkholder\textendash Davis\textendash Gundy}
\newacro{ConO}{consensus optimization}
\newacro{RM}{Robbins\textendash Monro}
\newacro{KW}{Kiefer\textendash Wolfowitz}
\newacro{GDA}{gradient descent/ascent}
\newacro{SGA}{symplectic gradient adjustment}
\newacro{SGD}{stochastic gradient descent}
\newacro{SGDA}{stochastic gradient descent/ascent}
\newacro{SPSA}{simultaneous perturbation stochastic approximation}
\newacro{ASGDA}[alt-SGDA]{alternating stochastic gradient descent/ascent}
\newacro{SEG}{stochastic extra-gradient}
\newacro{EG}{extra-gradient}
\newacro{PEG}{Popov's extra-gradient}
\newacro{RG}{reflected gradient}
\newacro{OG}{optimistic gradient}
\newacro{PPM}{proximal point method}

\newacro{GAN}{generative adversarial network}
\newacro{NN}{neural network}
\newacro{FTRL}{``follow the regularized leader''}
\newacro{CGD}{Competitive Gradient Descent}
\newacro{wp1}[w.p.$1$]{with probability $1$}

\definecolor{pastelblueold}{RGB}{56,146,236}
\definecolor{pastelblue}{RGB}{43,115,187}
\definecolor{pastelgreen}{RGB}{63,159,95}

\hypersetup{
  colorlinks=true,
  citecolor=pastelgreen,
  linkcolor=pastelblue,    % internal references (equations, figures, tables)
  urlcolor=pastelblue      % URLs
}

\title{Verifier-Constrained Flow Expansion \\ for Discovery Beyond the Data}

\author{
Riccardo {De Santi}\textsuperscript{$12$}\thanks{Equal contribution. Corresponding author: \texttt{rdesanti@ethz.ch}.}\;, Kimon Protopapas\textsuperscript{$1$}\footnotemark[1]\;, Ya-Ping Hsieh\textsuperscript{$1$}, Andreas Krause\textsuperscript{$12$} \\
\textsuperscript{$1$}ETH Zürich, \textsuperscript{$2$}ETH AI Center \\
\texttt{\{rdesanti,kprotopapas,krausea\}@ethz.ch}, \texttt{\{yaping.hsieh\}@inf.ethz.ch}
}

\begin{document}

\maketitle

\begin{abstract}
\looseness -1 Flow and diffusion models are typically pre-trained on limited available data (e.g., molecular samples), covering only a fraction of the valid design space (e.g., the full molecular space). As a consequence, they tend to generate samples from only a narrow portion of the feasible domain. This is a fundamental limitation for scientific discovery applications, where one typically aims to sample valid designs beyond the available data distribution. To this end, we address the challenge of leveraging access to a verifier (e.g., an atomic bonds checker), to adapt a pre-trained flow model so that its induced density expands beyond regions of high data availability, while preserving samples validity. We introduce formal notions of \emph{strong} and \emph{weak verifiers} and propose algorithmic frameworks for \emph{global} and \emph{local flow expansion} via probability-space optimization. Then, we present \AlgNameDef, a scalable mirror descent scheme that provably tackles both problems by verifier-constrained entropy maximization over the flow process noised state space. Next, we provide a thorough theoretical analysis of the proposed method, and state convergence guarantees under both idealized and general assumptions. Ultimately, we empirically evaluate our method on both illustrative, yet visually interpretable settings, and on a molecular design task showcasing the ability of \AlgNameShort to expand a pre-trained flow model increasing conformer diversity while preserving validity.
\end{abstract}

\addtocontents{toc}{\protect\setcounter{tocdepth}{-1}}

\section{Introduction} 
\label{sec:introduction}
\begin{wrapfigure}{r}{0.33\textwidth}
  \centering 
  \includegraphics[width=0.33\textwidth]{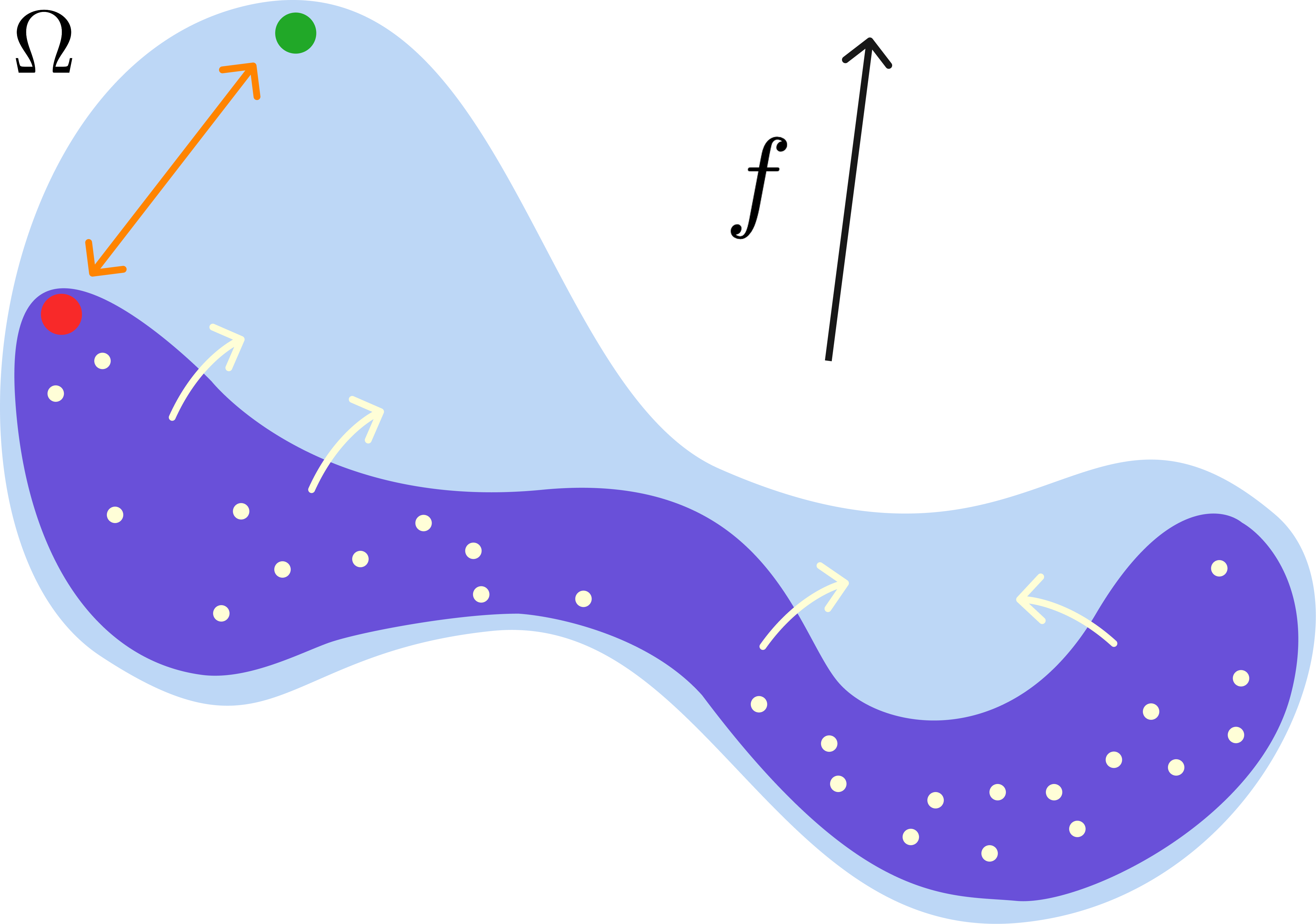}
  \caption{\looseness -1 Limited coverage of the valid design space leads to generating sub-optimal samples for downstream optimization tasks.}
  \label{fig:partial_manifold_coverage} 
\end{wrapfigure}

\looseness -1 Recent years have seen major progress in large-scale generative modeling. In particular, flow~\citep{lipman2022flow, lipman2024flow} and diffusion models~\citep{sohl2015deep, song2019generative, ho2020denoising} now produce high-fidelity samples and have been applied successfully across domains such as chemistry~\citep{hoogeboom2022equivariant}, biology~\citep{corso2022diffdock}, and robotics~\citep{chi2023diffusion}. These models are typically trained via divergence minimization objectives, such as score~\citep{song2020score} or flow matching~\citep{lipman2022flow}, to approximate the distribution induced by training data (\eg molecular samples), which typically only cover a tiny subset of the full valid design space. As a consequence, pre-trained generative models concentrate their density over a narrow region of valid designs, and are unlikely to generate valid samples beyond areas of high data availability. This is a fundamental limitation for scientific discovery tasks such as material design and drug discovery, where one typically wishes to generate valid designs beyond the data distribution. In particular, limited coverage of the valid design space leads to an irreducible sub-optimality gap in \emph{generative optimization}~\citep{santi2025flow, uehara2024feedback, li2024diffusion} problems, where one aims to generate samples maximizing a task-specific utility function $f: \X \to \R$ (\eg binding affinity for protein docking), as illustrated in Figure \ref{fig:partial_manifold_coverage}.

\looseness -1 Prior work has addressed this issue through manifold-exploration schemes that re-balance a pre-trained model’s density over diverse, promising modes~\citep[\eg][]{de2025provable, santi2025flow, celik2025dime}. However, the validity signal learned by a pre-trained flow model diminishes outside high-data regions. Therefore, seemingly promising low-probability modes that such methods would further explore, could turn out to be invalid. This highlights the need to inject further validity information into the exploration process via an external \emph{verifier}~\citep{botta2025query, wang2024multi}: formally, a function $v: \X \to \{0,1\}$ that provides data-specific validity signal. Luckily, there exists more-or-less accurate verifiers for a wide variety of real-world discovery applications, such as atomic-bond checkers for drug discovery~\citep[\eg][]{oboyle2011openbabel}, protein folding predictors for protein design~\citep[\eg][]{jumper2021highly}, as well as physics-based simulators for mechanical and material design~\citep[\eg][]{kresse1996efficient}. Motivated by these insights, in this work we advance flow and diffusion-based design space exploration methods~\citep{de2025provable, santi2025flow} by asking the following question:
\begin{center} 
   \emph{How can we leverage a given verifier to adapt a flow or diffusion model to generate designs beyond high data-availability regions while preserving validity?}
\end{center} 
\looseness -1 Answering this would contribute to the algorithmic-theoretical foundations of \emph{generative exploration}, and enable applications of flow-based exploration schemes in diverse scientific discovery tasks.

\paragraph{Our approach}  
\looseness -1 We address this challenge by formally introducing two verifier types. A \emph{strong verifier} is a function $v: \X \to \{0,1\}$ that characterizes validity exactly (\ie $v(x)=1$ iff $x$ is valid). A \emph{weak verifier} instead acts as a \emph{filter}: it rejects certain invalid designs but misses others (formally $v(x)=0 \implies$ $x$ is invalid). While the former is arguably rare in scientific discovery applications, the latter is ubiquitous. For instance, most molecular checkers examine specific constraints (\eg atomic bonds, graph topology, or conformer geometry), ruling out certain invalid samples, but without guaranteeing validity. We show that strong verifiers allow to adapt a pre-trained model to \emph{globally} expand over the entire valid design space. While this is not the case for weak verifiers, they can also be leveraged for a more conservative, \emph{local} expansion. To this end, we introduce mathematical frameworks for \emph{global} and \emph{local flow expansion} via verifier-constrained entropy maximization (Sec. \ref{sec:problem_setting}). Next, we propose \AlgNameDef, a scalable mirror descent scheme acting over the flow process noised state space that provably tackles both problems by sequentially alternating  expansion and projection steps (Sec. \ref{sec:algorithm}). We provide theoretical guarantees for \AlgNameShort, showing convergence results under both idealized and general assumptions via mirror-flow theory (Sec. \ref{sec:theory}). Ultimately, we evaluate our method on both illustrative, yet visually interpretable settings, and on a molecular design task, showcasing the ability of \AlgNameShort to expand a pre-trained flow model to increase molecular conformer diversity while better preserving validity than current flow-based exploration schemes (Sec. \ref{sec:experiments}).

\paragraph{{Our contributions}} In this work, we provide the following contributions:

\begin{itemize}[noitemsep,topsep=0pt,parsep=0pt,partopsep=0pt,leftmargin=*]
    \item \looseness -1 A formalization of \emph{Global} and \emph{Local Flow Expansion} via verifier-constrained entropy maximization, which formally capture the practically relevant problem of expanding the coverage of a pre-trained flow or diffusion model by integrating information from an available strong or weak verifier (Sec. \ref{sec:problem_setting}).
    \item \looseness -1 {\em \AlgNameDef}, a principled probability-space optimization scheme that provably solves both problems introduced via constrained entropy maximization over the flow noised state space (Sec. \ref{sec:algorithm}).
    \item \looseness -1  \AlgNameDefExplore{}, a noised state space unconstrained exploration scheme, obtained as a by product, that outperforms existing flow-based methods for high-dim. exploration (Sec. \ref{sec:algorithm}).
    \item A theoretical analysis of the proposed algorithm providing convergence guarantees under both simplified and realistic assumptions via mirror-flow theory (Sec. \ref{sec:theory}).
    \item \looseness -1 An experimental evaluation of \AlgNameShort, showcasing its practical relevance on both visually interpretable illustrative settings, and on a molecular design task aiming to increase conformer diversity. (Sec. \ref{sec:experiments}).
\end{itemize}

\section{Background and Notation} 
\label{sec:background}
\mypar{Mathematical Notation}
Using $\X \subseteq \R^d$ to refer to the design space (an arbitrary set), we denote the set of Borel probability measures on $\X$ with $\mP(\X)$, and the set of functionals over the set of probability measures $\mP(\X)$ as $\mF(\X)$. Given an integer $N$, we define $[N] \coloneqq \{1, \ldots, N\}$. 

\mypar{Flow-based Generative Modeling}
Generative models aim to approximately replicate and sample from a data distribution $p_{data}$. Flow models tackle this problem by modeling a \emph{flow}, which incrementally transforms samples $X_0 = x_0$ from a source distribution $p_0$ into samples $X_1=x_1$ from the target distribution $p_{data}$ \citep{lipman2024flow, farebrother2025temporal}. Formally, a \emph{flow} is a time-dependent map $\psi: [0,1]\times \R^d \to \R$ such that $\psi: (t,x) \to \psi_t(x)$. A \emph{generative flow model} is then defined as a continuous-time Markov process $\{X_t\}_{0 \leq t \leq 1}$ generated by applying a flow $\psi_t$ to $X_0$, i.e. $X_t = \psi_t(X_0), \ t \in [0, 1]$ such that $X_1 = \psi_1(X_0) \sim p_{data}$. In the context of flow modeling, the flow $\psi$ is defined by a \emph{velocity field} $u: [0,1] \times \R^d \to \R^d$, which is a vector field implicitly defining $\psi$ via the following ordinary differential equation (ODE), typically referred to as \emph{flow ODE}:
\begin{equation}
    \frac{\der}{\der t} \psi_t(x) = u_t(\psi_t(x)) \label{eq:flow_diff_eq} \, , \; \; \psi_0(x) = 0
\end{equation}
We write $\{p_t\}_{t \in [0, 1}$ to refer to the probability path of \emph{marginal densities} of the flow model, \ie $X_t = \psi_t(X_0) \sim p_t$, and denote by $p^u$ the probability path of marginal densities induced by the velocity field $u$. In practice, Flow Matching (FM)~\cite[][]{lipman2024flow} can be used to estimate a velocity field $u^{\theta}$ s.t. the probability path $p^{u_\theta}$ satisfies $p^{u_\theta}_0 = p_0$ and $p^{u_\theta}_1 = p_{data}$, where $p_0$ denotes the source distribution, and $p_{data}$ the target data distribution. Typically FM is rendered tractable by defining $p^u_t$ as the marginal of a conditional density $p^u_t(\cdot | x_0, x_1)$, \eg:
\begin{equation}
    X_t \; |\; X_0, X_1 = \kappa_t X_0 + \omega_t X_1
\end{equation}
\looseness -1 where $\kappa_0 = \omega_1 = 1$ and $\kappa_1 = \omega_0 = 0$ (e.g., $\kappa_t = 1- t$ and $\omega_t = t$). Then $u^\theta$ can be learned by regressing onto the conditional velocity field $u(\cdot | x_1)$ \citep{lipman2022flow}. 
Interestingly, diffusion models~\citep{song2019generative} (DMs) admit an equivalent ODE‐based formulation with identical marginal densities to their original SDE dynamics~\citep[Chapter 10]{lipman2024flow}. Consequently, while in this work we adopt the notation of flow models, our contributions carry over directly to DMs.

\mypar{Continuous-time Reinforcement Learning}
We formulate continuous-time reinforcement learning (RL) as a specific class of finite-horizon optimal control problems \citep{wang2020reinforcement, jia2022policy, treven2023efficient, zhao2024scores}. Given a state space $\X$ and an action space $\A$, we consider the transition dynamics governed by the following ODE:
\begin{equation} 
      \frac{\der}{\der t} \psi_t(x) = a_t(\psi_t(x)) \label{eqn_continuous_RL} 
\end{equation}
\looseness -1 where $a_t \in \A$ is a selected action. We consider a state space $\X \coloneqq \R^d \times [0,1]$, and denote by (Markovian) deterministic policy a function $\pi_t(X_t) \coloneqq \pi(X_t, t) \in \A$ mapping a state $(x,t) \in \X$ to an action $a \in \A$ such that $a_t = \pi(X_t, t)$, and denote with $p_t^\pi$ the marginal density at time $t$ induced by policy $\pi$. Considering the continuous-time reinforcement learning formulation above, a velocity field $u^{pre}$ can be interpreted as an action process $a^{pre}_t \coloneqq u^{pre}(X_t, t)$, where $a_t^{pre}$ is determined by a continuous-time RL policy via $a_t^{pre} = \pi^{pre}(X_t, t)$~\citep{de2025provable}. Therefore, we can express the flow ODE induced by a pre-trained flow model by replacing $a_t$ with $a^{pre}$ in Eq. \ref{eqn_continuous_RL}, and denote the pre-trained model by its policy $\pi^{pre}$, which induces a density $p^{pre}_1 \coloneqq p^{\pi^{pre}}_1$ approximating $p_{data}$.

\section{Problem Statement: Global and Local Flow Expansion}
\label{sec:problem_setting}
\begin{figure*}[t]
    \centering
    \begin{subfigure}{0.55\textwidth} 
        \centering
        \includegraphics[width=\textwidth, keepaspectratio]{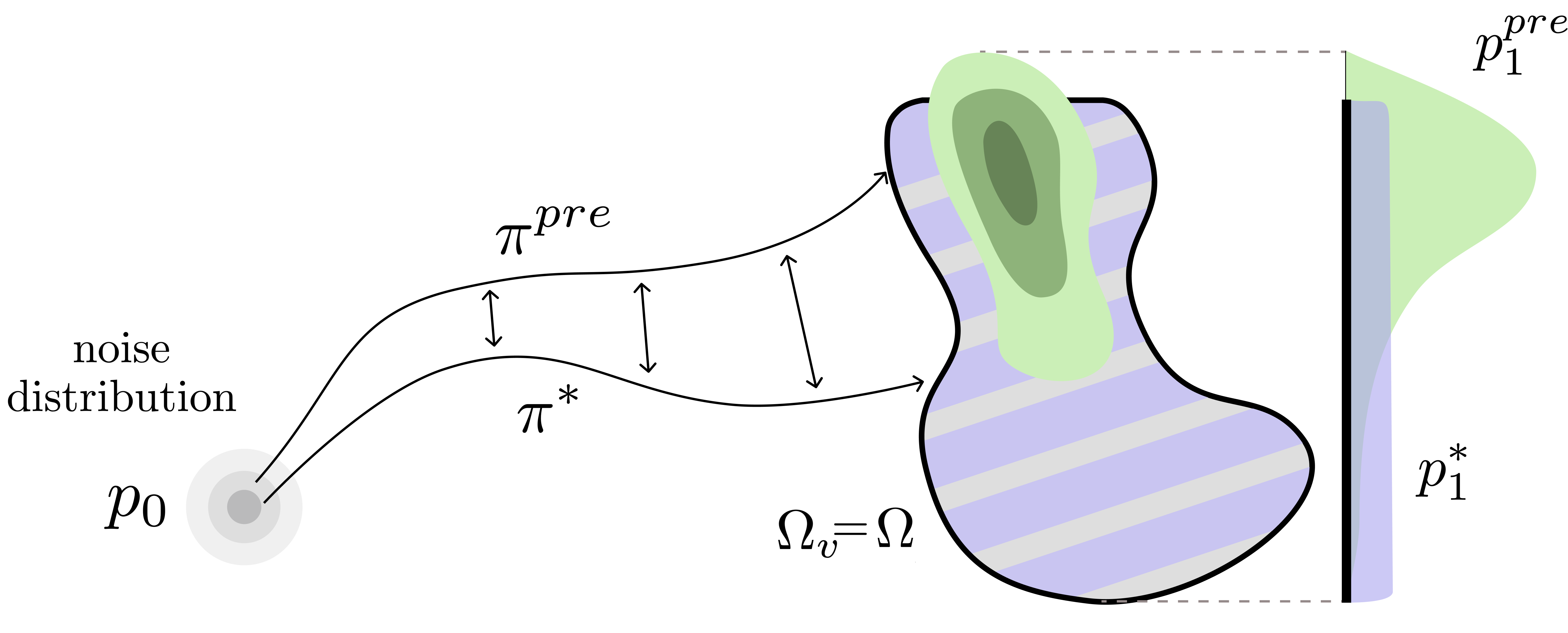}
        \caption{Global flow expansion}
        \label{fig:process_drawing}
    \end{subfigure}%
    \hspace{15pt}
    \begin{subfigure}{0.3\textwidth} 
        \centering
        \raisebox{-1.3ex}[0pt][0pt]{
        \includegraphics[width=\textwidth, keepaspectratio]{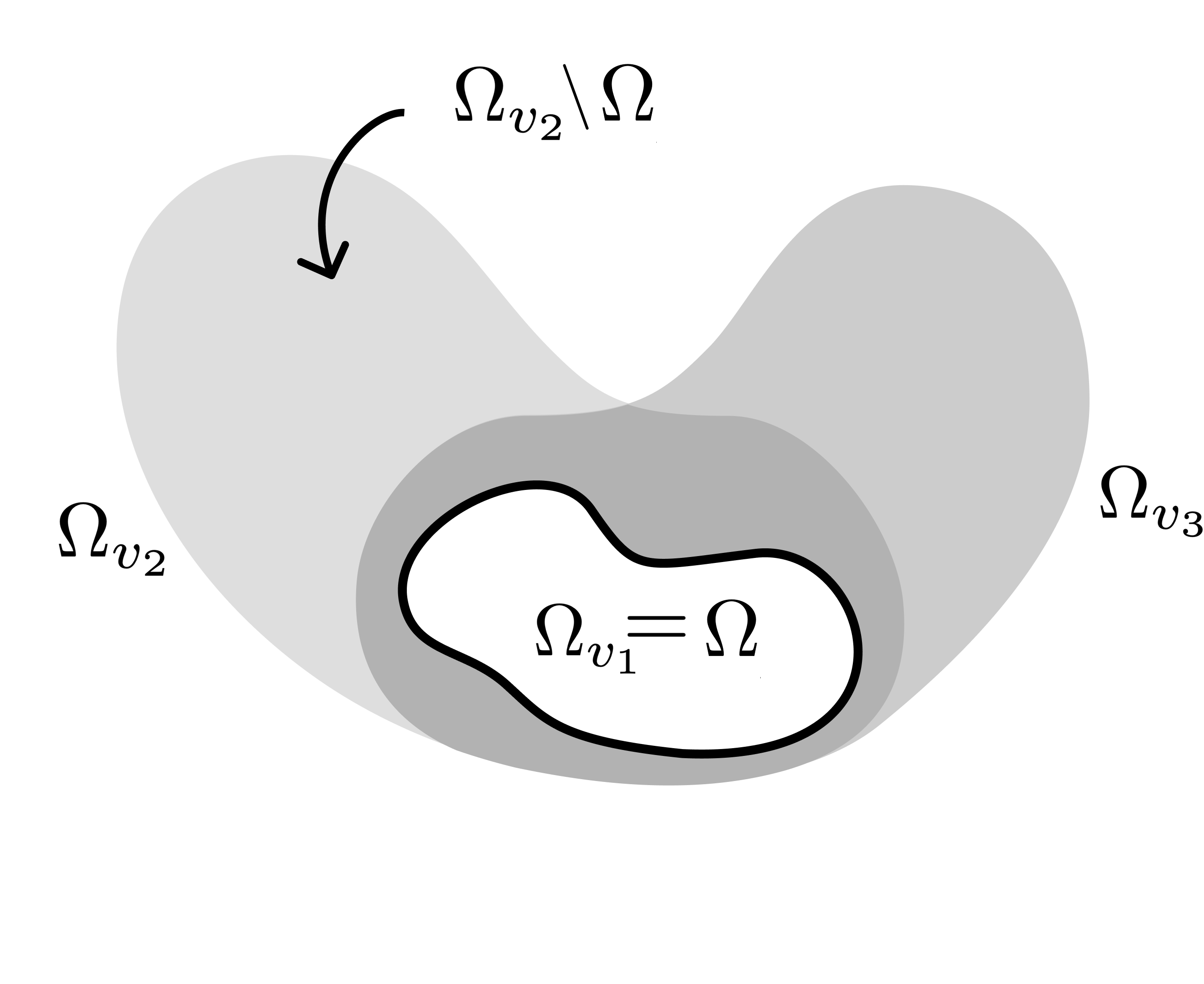}
        }
        \caption{Strong and weak verifier sets}
        \label{fig:verifiers_drawing}
    \end{subfigure}
    \caption{\looseness -1 (\ref{fig:process_drawing}) Pre-trained and globally expanded flow model inducing densities $p^{pre}_1$ and optimal density $p_1^*$.  (\ref{fig:verifiers_drawing}) Valid design space $\Omega$, strong and weak verifiers $\Omega_{v_i}$, $i \in [3]$, and their compositions.}
    \label{fig:joint_top_figures} 
\end{figure*}

\looseness -1 Given a pre-trained flow model $\pi^{pre}$ inducing a density $p_1^{pre}$ that covers sufficiently only a limited region of the \emph{valid design space}\footnote{In this work, we consider the valid design space to be an unknown, yet bounded set.} $\Omega \subseteq \R^d$ (\eg a molecular space, see Fig. \ref{fig:partial_manifold_coverage}), we aim to adapt it by leveraging a verifier (\eg an atomic bonds checker) to compute a model $\pi^*$, inducing a process: 
\begin{equation} 
      \frac{\der}{\der t} \psi_t(x) = a^*_t(\psi_t(x)) \quad  \text{with} \quad a_t^* = \pi^*(x,t), \label{eq:controlled_ODE} 
\end{equation}
\looseness -1 such that its density $p_1^* \coloneqq p_1^{\pi^*}$ is more uniformly distributed over the valid design space than \smash{$p_1^{pre}$}. To this end, we first denote by \emph{verifier} a scalar function $v: \X \to \{0,1\}$, and indicate by $\Omega_v = \{x \in \X : v(x) = 1\}$ the \emph{verifier-set} induced by $v$. Next, we classify any verifier $v$ as \emph{strong} or \emph{weak} depending on the relationship between its verifier-set $\Omega_v$ and the valid design space $\Omega$. 

\begin{restatable}[Strong Verifier]{definition}{strongVerifer}
\label{definition:strong_verifier}
We denote by \emph{strong verifier} a function \smash{$v: \X \to  \{0,1\}$} s.t. $\Omega_v = \Omega$.
\end{restatable}
\looseness -1 By Def. \ref{definition:strong_verifier}, $v(x) = 1 \iff x \in \Omega$, hence a strong verifier fully characterizes the valid design space $\Omega$.

\subsection{An Idealized Problem: Global Flow Expansion via Strong Verifiers}
Given a pre-trained flow model $\pi^{pre}$ and a  strong verifier $v: \X \to \{0,1\}$ as defined within Def. \ref{definition:strong_verifier}, one can capture the problem of computing a new flow model $\pi^*$ such that its marginal density \smash{$p_1^* \coloneqq p_1^{\pi^*}$} covers $\Omega$ uniformly via the following verifier-constrained entropy maximization problem.
\begin{tcolorbox}[colframe=white!, top=2pt,left=2pt,right=2pt,bottom=2pt]
\begin{center}
\textbf{Global Flow Expansion via Verifier-Constrained Entropy Maximization}
\begin{align}
      \pi^* \in \argmax_{\pi : p_0^* = p_0^{pre}}\;  \mathcal{H}(p_1^\pi)  \quad \text{ subject to } \quad  s.t. \quad p^{\pi}_1 \in \mP(\Omega_{v}) \label{eq:global_flow_expansion_problem}
\end{align}
\end{center}
\end{tcolorbox}
In this formulation, the constraint $p_0^\pi = p_0^{pre}$ enforces that the marginal density at $t=0$ matches the pre-trained model marginal, and $\entropy \in \mF(\Omega_{v})$ denotes the differential entropy functional expressed as: 
\begin{equation}
 \entropy(\mu) = - \int \der \mu \log \frac{\der \mu}{dx}, \quad \mu \in \mP(\Omega_{v}) \label{eq:entropy_functional_def}
\end{equation}
\looseness -1 where $\Omega_v$ is the bounded verifier-set induced by $v$. 
Crucially, Problem \ref{eq:global_flow_expansion_problem} computes a flow model $\pi^*$ inducing the density $p_1^{\pi^*}$ with maximum entropy among all densities supported on $\Omega_v$. Therefore, the optimal density $p_1^{\pi^*}$ according to Problem \ref{eq:global_flow_expansion_problem} corresponds to the uniform density over the entire valid design space $\Omega$, \ie $p_1^{\pi^*} = \mathcal{U}(\Omega)$ - as uniforms are the entropy-maximizing densities on bounded sets and $\Omega_v = \Omega$ due to $v$ being a strong verifier. 
Notably, Problem \ref{eq:global_flow_expansion_problem} does not depend on the prior generative model $\pi^{pre}$. In fact, since the strong verifier $v$ fully characterizes the valid design space $\Omega$, prior information is not required to compute the maximally explorative, yet valid flow model $\pi^*$. 

\looseness -1 Problem \ref{eq:global_flow_expansion_problem} provides a sharp data-free objective for verifier-based flow/diffusion model learning, well capturing the ideal goal of a uniform prior over the valid design space for subsequent use in downstream tasks. Nonetheless, as discussed in Sec. \ref{sec:introduction}, strong verifiers are arguably rare in most scientific discovery applications (\eg material design, drug discovery). Towards overcoming such limitation, in the following we sharpen our notion of verifier to one that is ubiquitous in real-world discovery tasks.
\subsection{A Realistic Framework: Local Flow Expansion via Weak Verifiers}
\looseness -1 We first relax the notion of strong verifier introduced in Def. \ref{definition:strong_verifier} to the following one of \emph{weak verifier}.
\begin{restatable}[Weak Verifier]{definition}{weakVerifer}
\label{definition:weak_verifier}
\looseness -1 We denote by \emph{weak verifier} a function $v: \X \to \{0,1\}$ s.t. $\Omega_v \supset \Omega$.
\end{restatable}
\looseness -1 As Fig. \ref{fig:verifiers_drawing} illustrates, Def. \ref{definition:weak_verifier} requires only the one-sided condition $v(x) = 0 \implies x \notin \Omega$; unlike strong verifiers, it does not guarantee $v(x) = 1 \implies x \in \Omega$ (\ie $v$ cannot certify membership in $\Omega$). Instead, it represents a superset $\Omega_v \supset \Omega$ and effectively acts as a \emph{filter}. Moreover, multiple weak verifiers $\{v_i\}$ can be combined, yielding $\Omega_v = \bigcap_i \Omega_{v_i}$, which is typically tighter to $\Omega$ for more diverse verifiers $v_i$, \eg checking atomic bonds, molecular graph topology, and conformer geometry.  

\looseness -1 Given this new realistic notion of verifier, the global flow expansion Problem \ref{eq:global_flow_expansion_problem} would evidently no longer compute the desired flow model. In fact, for a weak verifier $v$ it holds $\Omega_v \supset \Omega$,  therefore the optimal flow density $p^* = \mathcal{U}(\Omega_v)$ would generate invalid designs over $\Omega_v \setminus \Omega$, as shown in Fig. \ref{fig:verifiers_drawing}. Moreover, weak verifiers typically induce unbounded verifier sets, which would even render Problem \ref{eq:global_flow_expansion_problem} ill-posed. 
To address these issues, we introduce the \emph{local flow expansion} problem, which aims to locally expand the prior flow model $\pi^{pre}$ by integrating information from both $v$ and $\pi^{pre}$.
\begin{tcolorbox}[colframe=white!, top=2pt,left=2pt,right=2pt,bottom=2pt]
\begin{center}
\textbf{Local Flow Expansion via KL-regularized Verifier-Constrained Entropy Maximization}
\begin{align}
      \pi^* \in \argmax_{\pi : p_0^* = p_0^{pre}}\;  \mathcal{H}(p_1^\pi) - \alpha \D_{KL}(p_1^\pi \| p_1^{pre})  \quad \text{ subject to } \quad  s.t. \quad p^{\pi}_1 \in \mP(\Omega_{v})  \label{eq:local_flow_expansion_problem}
\end{align}
\end{center}
\end{tcolorbox}
\looseness -1 Here, the weak verifier $v$ acts as a filter preventing the entropy term from driving exploration into verifier-rejected regions. Since $v$ cannot detect all invalid areas, expansion must remain conservative and leverage the validity signal encoded in the prior model. This is achieved via the $\alpha$-weighted KL divergence between the density $p_1^\pi$ induced by the fine-tuned model, and $p_1^{pre}$. Crucially, this term enforces $\pi^*$ to preserve prior validity signal, thus preventing $\pi^*$ from allocating density in regions unlikely according to $\pi^{pre}$, even if valid according to the weak verifier. For sufficiently large $\alpha$, the density induced by the expanded flow model $\pi^*$ stays arbitrarily close to the prior in probability space - hence \emph{local} expansion. In practice, the choice of $\alpha$ should reflect the degree of risk-aversion versus novelty-seeking toward the discovery task at hand, as well as the quality of the weak verifier $v$ (\ie how tightly $\Omega_v$ approximates $\Omega$). Interestingly, in the limit of $\Omega_v \to \Omega$, $\alpha$ should clearly be set to $0$, which naturally retrieves the presented global flow expansion Problem \ref{eq:global_flow_expansion_problem} as a sub-case.

\section{\AlgNameLong: Scalable Global and Local Expansion via Verifier-Constrained Noised Space Entropy Maximization} 
\label{sec:algorithm}
\looseness -1 In the following, we propose \AlgNameDef, which provably solves the global and local flow expansion problems (see Eq. \ref{eq:global_flow_expansion_problem} and \ref{eq:local_flow_expansion_problem}). To this end, we first lift their formulations from the probability space associated to the last time-step marginal $p_1^\pi$ to the entire flow process \smash{$\mathbf{Q}^\pi = \{p^\pi_t\}_{t\in [0,1]}$}.
\begin{tcolorbox}[colframe=white!, top=2pt,left=2pt,right=2pt,bottom=2pt]
\begin{center}
\textbf{Flow Expansion via Verifier-Constrained Noised Space Entropy Maximization}
\begin{align}
      \pi^* \in \argmax_{\pi : p_0^\pi = p_0^{pre}}\;  \Lfunc \left(\mathbf{Q}^\pi\right) \coloneqq \int_0^1  \lambda_t \mathcal{G}_t(p_t^\pi) \; \der t \quad \text{ subject to } \quad  \EV_{x \sim p_1^\pi}[v(x)] = 1  \label{eq:noised_flow_expansion_problem}
\end{align}
\end{center}
\end{tcolorbox}
\looseness -1 Under this unifying formulation, $\G_t: \mP(\X) \to \R$ is a functional over densities $p^\pi_t$ induced by flow $\pi$. We note that under general regularity assumptions, an optimal policy $\pi^*$ for Problem \ref{eq:noised_flow_expansion_problem} is optimal also for the global and local flow expansion problems (see Eq. \ref{eq:global_flow_expansion_problem} and \ref{eq:local_flow_expansion_problem}) if the functional $\G$ is defined as:
\begin{equation}
\begin{aligned}
\underbrace{\mathcal{G}_t(p_t^\pi) = \mathcal{H}(p_t^\pi)}_{\text{Global Flow Expansion}}
\qquad
\underbrace{\mathcal{G}_t(p_t^\pi) = \mathcal{H}(p_t^\pi) - \alpha_t \,\mathrm{D}_{\mathrm{KL}}\!\big(p_t^\pi \,\Vert\, p_t^{\mathrm{pre}}\big)}_{\text{Local Flow Expansion ($\alpha_1 = \alpha$)}}
\end{aligned}
\end{equation}
\looseness -1 Before introducing \AlgNameShort, we first recall the standard notion of first variation of $\G$ over a space of probability measures~\citep[cf.][]{hsieh2019finding}. A functional $\G \in \mF(\X)$, where $\G: \mP(\X) \to \R$, has first variation at $\mu \in \mP(\X)$ if there exists a function $\delta \G(\mu) \in \mF(\X)$ such that for all $\mu' \in \mP(\X)$ it holds that:
\begin{equation*}
    \G(\mu + \epsilon \mu') = \G(\mu) + \epsilon \langle \mu', \delta \G(\mu) \rangle + o(\epsilon). 
\end{equation*}
\looseness -1 where the inner product is an expectation. Intuitively, $\delta \G(\mu)$ can be interpreted as an infinite-dimensional gradient over probability measures. Given this concept of first variation, \AlgNameShort solves Problem \ref{eq:noised_flow_expansion_problem} by computing a process $\mathbf{Q}^{k}$ at each iteration $k \in [K]$, via the following mirror descent step:
\begin{tcolorbox}[colframe=white!, top=2pt,left=2pt,right=2pt,bottom=2pt]
\begin{center}
\textbf{(MD Step) Constrained and Regularized Process Surprise Maximization}
\begin{equation}
      \mathbf{Q}^{k} \in \argmax_{\mathbf{Q} : p_0 = p^{k-1}_0} \langle \delta \Lfunc(\mathbf{Q}^{k-1}), \mathbf{Q}\rangle - \frac{1}{\step^{k}} D_{KL}\left(  \mathbf{Q}  \| \mathbf{Q}^{k-1} \right) \; \text{s.t.} \;  \EV_{x \sim p_1}[v(x)] = 1  \label{eq:MD_step}
\end{equation}
\end{center}
\end{tcolorbox}

\looseness -1 While the MD step in Eq. \ref{eq:MD_step} might seem abstract, the following Lemma \ref{lemma:first_var_process} hints at a more practical formulation of the above through the lens of stochastic optimal control~\citep{fleming2012deterministic}.

\begin{restatable}[First Variation of Flow Process Functionals]{lemma}{FirstVarProcess}
\label{lemma:first_var_process}
For objectives defined in the form of Eq. \ref{eq:noised_flow_expansion_problem}, we have: 
    \begin{equation}
    \langle \delta \Lfunc(\mathbf{Q}^k), \mathbf{Q} \rangle =   \int_0^1 \lambda_t \EV_{\quad \mathbf{Q}} \left[ \delta \G_t(p^k_t)  \right]  \der t. 
    \end{equation}
\end{restatable}

\looseness -1 Lemma \ref{lemma:first_var_process} factorizes $\langle \delta \mathcal{L}(\mathbf{Q}^{k - 1}), \mathbf{Q}\rangle$ into an integral over the flow process of terms $f_t(x) := \lambda_t \delta \mathcal{G}_t(p^k_t)(x)$. Crucially, this time-decomposition allows to rewrite the MD step (Eq. \ref{eq:MD_step}) as the following standard constrained control-affine optimal control problem\footnote{We leave standard dynamical system constraints~\citep[\eg Equation 13][]{domingo2024adjoint} as implicit.}~\citep{domingo2024adjoint}:
\begin{tcolorbox}[colframe=white!, top=2pt,left=2pt,right=2pt,bottom=2pt]
\begin{center}
\textbf{Constrained and Regularized Process Surprise Maximization via Fine-Tuning}
\begin{align*}
    \min_{\pi} \; \EV \left[ \int_0^1 \frac{1}{2}\|\pi(X_t,t)\|^2 - f_t(X_t,t) \; \der t \right] s.t. \EV_{x \sim p_1}[v(x)] = 1, \text{ with } f_t(X_t,t) = \gamma_t \delta \G_t(p_t^k)(x)
\end{align*}
\end{center}
\end{tcolorbox}

\begin{algorithm}[t]
\caption{\ExpandThenProject}
\begin{algorithmic}[1]
    \State{\textbf{Input:} $\pi^{k - 1}$: current flow model,$\nabla_{x_t}\delta \mathcal{G}$: gradients of functional grad., $\gamma_k$: inverse update step-size, $\{\lambda_t\}_{t \in [0, 1]}$: integral weighting coefficients , $v$: verifier, $\eta_k$: fine-tuning strength}
    \State{\textbf{Expansion} step:
    \begin{equation}
        \Tilde{\pi}^{k} \leftarrow \FineTuningSolver(\pi^{k-1}, \nabla_{x_t}\delta \mathcal{G}_t, \lambda_t, \gamma_k)
    \end{equation}
    }
    \State{\textbf{Projection} step:
    \begin{equation}
       \pi^{k} \leftarrow \FineTuningSolver(\Tilde{\pi}^{k}, \log v, \eta_k)
    \end{equation}
    } 
    \State{\textbf{Output:} Fine-tuned policy $\pi^k$}
\label{alg:expand_then_project}
\end{algorithmic}
\end{algorithm}

\begin{algorithm}[t] %H or t
    \small
    \caption{\AlgNameDef} 
    \label{alg:algorithm}
        \begin{algorithmic}[1]
        \State{\textbf{Input: } \looseness -1 $\pi^{pre}: $ pre-trained flow model, $\{\alpha_t\}_{t \in [0,1]}:$ KL-regularization coefficients, $\{\gamma_k \}_{k=1}^{K}:$ inverse update step-sizes, $\{\lambda_t \}_{t \in [0,1]}:$  integral weighting coefficients, $v$: verifier, $\{\eta_k\}_{k = 1}^K$}: projection strength schedule
        \State{\textbf{Init:} $\pi_0 \coloneqq \pi^{pre} $}
        \For{$k=1, 2, \hdots, K$}
        \State{Set: 
        \begin{equation}
        \begin{aligned}
        \nabla_{x_t}\delta\mathcal{G}_t(p_t^{k-1}) &=
        \left\{
        \begin{aligned} \label{eq:running-cost}
        &- s^{\pi^{k-1}}_t
        &&  \quad \text{\AlgNameDefGlobal} \\
        &- s^\pi_t - \alpha_t\left(s^\pi_t - s^\textrm{pre}_t\right)
        && \quad \text{\AlgNameDefLocal}
        \end{aligned}
        \right.
        \end{aligned}
        \end{equation}
        }
        \State{Fine-tune $\pi_{k-1}$ into $\pi_k$ via Algorithm \ref{alg:expand_then_project}:
            \begin{equation*}
                \pi_k \leftarrow \ExpandThenProject(\pi^{k - 1}, \nabla_{x_t}\delta \mathcal{G}_t, \gamma_k, \{\lambda_t\}_{t \in [0,1]}, v, \eta_k) 
            \end{equation*}
        }
        \EndFor
        \State{\textbf{Output: } policy $\pi \coloneqq \pi_{K}$} 
        \end{algorithmic}
\end{algorithm}

\looseness -1 
\looseness -1 Concretely, we compute a flow $\pi^k$ inducing $\mathbf{Q}^k$ (Eq. \ref{eq:MD_step}) via \ExpandThenProject (see Alg. \ref{alg:expand_then_project}),  which decouples constrained optimization into sequential expansion and projection steps:
\paragraph{Expansion step.}
The unconstrained expansion step is performed over the noised state space, which can be tackled by extending established control (or RL) based methods for fine-tuning with the running cost $f_t(X_t,t) = \gamma_t \delta \G_t(p_t^k)(x)$, effectively computing a process $\Tilde{\mathbf{Q}}^k$ such that:
\begin{equation}\label{eq:main_expansion}
    \Tilde{\mathbf{Q}}^k \in \argmax_{\mathbf{Q} : p_0 = p^\textrm{pre}_0 } \langle \delta \mathcal{L}(\mathbf{Q}^{k - 1}), \mathbf{Q}\rangle - \frac{1}{\gamma_k} D_{\textrm{KL}}(\mathbf{Q} || \mathbf{Q}^{k - 1})
\end{equation}
\paragraph{Projection step.}
Given $\Tilde{\mathbf{Q}}^k$, the projection step adapts the flow $\Tilde{\pi}^k$ to enforce the constraint in Eq. \ref{eq:MD_step} via reward-guided fine-tuning~\citep[\eg][Sec. 8.2]{uehara2024understanding}:
\begin{equation}\label{eq:main_proj}
     \mathbf{Q}^k \in \argmax_{\mathbf{Q} : p_0 = p^\textrm{pre}_0 } \EV_{x \sim p_1} \left[\log v(x)\right] - D_{\textrm{KL}}(\mathbf{Q} || \Tilde{\mathbf{Q}}^k) 
\end{equation}

This \ExpandThenProject scheme provably computes the optimal flow for the MD step in Eq. \ref{eq:MD_step}.
\begin{proposition} \label{prop:expandthenproject}
    The \ExpandThenProject scheme in Alg. \ref{alg:expand_then_project} solves optimization problem \ref{eq:MD_step}, \ie it returns a flow model $\pi^k$ inducing a process $\mathbf{Q}^k$ that is a solution to \ref{eq:MD_step}. 
    Formally, the following holds:
    \begin{equation}
        \mathbf{Q}^k \in \argmax_{\mathbf{Q} : p_0 = p^{k-1}_0} \langle \delta \Lfunc(\mathbf{Q}^{k-1}), \mathbf{Q}\rangle - \frac{1}{\step^{k}} D_{KL}\left(\mathbf{Q}  \| \mathbf{Q}^{k-1} \right) \; \text{s.t.} \;  \EV_{x \sim p_1}[v(x)] = 1
    \end{equation}
\end{proposition}
Notice that this step could alternatively be performed via constrained reward-guided fine-tuning~\citep[\eg][]{gutjahr2025constrained}.
Finally, we present \AlgNameDef $ $ in Alg. \ref{alg:algorithm}, which effectively approximates the mirror descent scheme presented above by iteratively applying \ExpandThenProject. 

\paragraph{Complete algorithm execution.} 
\looseness -1 In practice, \AlgNameShort takes as inputs: a pre-trained flow model \smash{$\pi^{pre}$}, KL-regularization coefficients $\{\alpha_t\}_{t\in [0,1]}$, the number of iterations $K$, inverse step-sizes \smash{$\{\gamma_k \}_{k=1}^{K}$}, integral weighting coefficients \smash{$\{\lambda_t\}_{t\in [0,1]}$}, a verifier $v$, and a projection strength schedule \smash{$\{\eta_k\}_{k=1}^K$}. At each iteration $k \in [K]$, \AlgNameShort computes the gradient (\wrt $x$) of the first variation at the previous policy $\pi_{k-1}$, \ie \smash{$\nabla_x \delta \G_t ( p^{\pi^{k-1}}_1)$} (line $4$). Then, \AlgNameShort computes the flow model $\pi_k$ via the \ExpandThenProject scheme (see Alg. \ref{alg:expand_then_project}), which takes in input the current flow model $\pi^{k-1}$, the computed gradients, and the verifier $v$ - and returns the updated flow model $\pi^k$. Ultimately, \AlgNameShort returns the final policy $\pi \coloneqq \pi_K$.

\looseness -1 \paragraph{Closed-form gradient expressions.} \AlgNameShort operates using trajectory reward gradients $\nabla_{x_t} \delta \mathcal{G}_t(p^\pi_t)$. In fact, while such rewards are difficult to estimate, their gradients admit close-form expressions~\citep{de2025provable} that can be approximated via available quantities:
\begin{equation}
\begin{aligned}
\underbrace{\nabla_x \delta \mathcal{H}(p_t^\pi) = -s_t^{\pi}}_{\text{Global Flow Expansion}}
\qquad
\underbrace{\nabla_{x_t}\delta\mathcal{H}(p_t^\pi)
        - \alpha_t \, \nabla_x \delta \mathrm{D}_{\mathrm{KL}}\!\big(p_t^\pi \,\Vert\, p_t^{\mathrm{pre}}\big) = - s^{\pi}_t - \alpha_t\left(s^{\pi}_t - s^\textrm{pre}_t\right)}_{\text{Local Flow Expansion ($\alpha_1 = \alpha$)}} \label{eq:grad_expressions}
\end{aligned}
\end{equation}

\looseness -1 and can simply be plugged into any first-order fine-tuning solver yielding a scalable method. The gradients in Eq. \ref{eq:grad_expressions} are expressed in terms of the \textit{score function} $s^\pi_t(x) = \nabla \log p^\pi_t(x)$, which can be approximated via the score network in the case of diffusion models~\citep{de2025provable}, and expressed via a linear transformation of the learned velocity field for flows~\citep{domingo2024adjoint}:
\begin{equation}\label{eq:score-transform}
    s^\pi_t(x) = \frac{1}{\kappa_t(\tfrac{\dot\omega_t}{\omega_t}\kappa_t - \dot\kappa_t)}\left(\pi(x, t)  - \frac{\dot\omega_t}{\omega_t}x\right) \,  
\end{equation}

Prior work for flow-based exploration relies only on the terminal score $s_1^\pi$~\citep{de2025provable, santi2025flow}. Nonetheless, by Eq. \ref{eq:score-transform} the score diverges as $t \to 1$, creating instabilities. While this can be partially managed by approximating $s_1^\pi \approx s^\pi_{1 - \epsilon}$ for $\epsilon > 0$, determining the correct $\epsilon$ can be challenging in practice. Our algorithm, by leveraging score information along the entire flow process, offers a natural and principled solution to this issue by choosing small $\lambda_t$ (\eg $\lambda_t = 0$, see Eq. \ref{eq:noised_flow_expansion_problem}) for $t \to 1$.

\looseness -1 {Beyond verifier-constrained settings, the algorithmic idea of noised space exploration introduced above also yields an improved \emph{unconstrained} exploration scheme. We denote by \AlgNameDefExplore{} the unconstrained algorithm obtained from \AlgNameShort{} by simply removing the projection step (line~17 of Alg.~\ref{alg:expand_then_project}). For completeness, we report the pseudocode of \AlgNameShortExplore{} in Alg.~\ref{alg:algorithm_nse}. As shown empirically in Sec.~\ref{sec:experiments}, \AlgNameShortExplore{} stabilizes diffusion and flow-based exploration  in higher-dimensional settings, leading to significantly better performance than existing methods \eg \citep{santi2025flow}.

\begin{algorithm}[t] %H or t
    \small
    \caption{ \AlgNameDefExplore } 
    \label{alg:algorithm_nse}
        \begin{algorithmic}[1]
        \State{\textbf{Input: } \looseness -1 $\pi^{pre}: $ pre-trained flow model, $\{\alpha_t\}_{t \in [0,1]}:$ KL-regularization coefficients, $\{\gamma_k \}_{k=1}^{K}:$ inverse update step-sizes, $\{\lambda_t \}_{t \in [0,1]}:$  integral weighting coefficients}
        \State{\textbf{Init:} $\pi_0 \coloneqq \pi^{pre} $}
        \For{$k=1, 2, \hdots, K$}
        % \State{Compute $p^{k-1}_1$ via Ito density estimation}
        \State{Set: 
        \begin{equation}
        \nabla_{x_t}\delta\mathcal{G}_t(p_t^{k-1}) = - s^\pi_t - \alpha_t\left(s^\pi_t - s^\textrm{pre}_t\right)
        \end{equation}
        }
        \State{\textbf{Expansion} step, fine-tune $\pi_{k-1}$ into $\pi_k$ via:
        \begin{equation}
            \pi^{k} \leftarrow \FineTuningSolver(\pi^{k-1}, \nabla_{x_t}\delta \mathcal{G}_t, \lambda_t, \gamma_k)
        \end{equation}
        }
        \EndFor
        \State{\textbf{Output: } policy $\pi \coloneqq \pi_{K}$} 
        \end{algorithmic}
\end{algorithm}

\section{Guarantees for \AlgNameLong}
\label{sec:theory}
\newcommand{\Qmd}{\mathbf{Q}_\sharp}

\looseness -1 We aim to show that \AlgNameShort admits \emph{provable guarantees} ensuring reliable behavior in practice. To this end, we leverage the flexible framework of \emph{constrained mirror descent}, a classical optimization method that has recently found successful applications in sampling and generative modeling \citep{karimi2024sinkhorn, de2025provable, santi2025flow}. We analyze two regimes. First, an \textbf{idealized setting}, where each step of Eq. \ref{eq:MD_step} can be computed \emph{exactly} - leading to sharp step-size prescriptions and fast, polynomial convergence rates. Then, a \textbf{realistic setting}, where each MD step can only be solved \emph{approximately} - for which we show asymptotic convergence to the optimal solution under mild noise and bias assumptions.

\paragraph{Idealized setting.}  
\looseness -1 We state that the exact updates case admits finite-time convergence guarantee:  
\begin{tcolorbox}[colframe=white!, top=2pt,left=2pt,right=2pt,bottom=2pt]
\begin{restatable}[Convergence guarantee in the idealized process-level setting]{theorem}{trajConvexCase}
\label{theorem:convex_case_convergence}
Consider the objective $\Lfunc$ defined in \cref{eq:noised_flow_expansion_problem}, and let 
$\lambda^\star \coloneqq \int_0^1 \lambda_t\mathrm{d}t$.  
Let $\{\mathbf{Q}^{k}\}$ be the iterates generated by \cref{eq:MD_step} with 
$\gamma_k = 1/\lambda^\star$ for all $k \in [K]$. Then
\begin{equation}
    \Lfunc(\mathbf{Q}^*) - \Lfunc(\mathbf{Q}^K) 
    \leq \frac{\lambda^\star}{K} \, D_{KL}\!\left(\mathbf{Q}^* \,\|\, \mathbf{Q}^{pre}\right),
\end{equation}
where $\mathbf{Q}^* \in \argmax_{\mathbf{Q}} \Lfunc(\mathbf{Q})$.
\end{restatable}
\end{tcolorbox}
\paragraph{General setting.}  
Recall that $\mathbf{Q}^k$ is the $k$-th iterate of \AlgNameShort.  
In realistic scenarios, however, Eq. \ref{eq:MD_step} can only be solved approximately, so we interpret the update as \emph{approximating} the idealized iteration:  
\begin{equation}
    \mathbf{Q}^k_\sharp \in \argmax_{\mathbf{Q} : p_0 = p^{k-1}_0} 
     \, \langle \delta \Lfunc(\mathbf{Q}^{k-1}), \mathbf{Q}\rangle 
    - \tfrac{1}{\gamma^{k}} D_{KL}\!\left(\mathbf{Q} \,\|\, \mathbf{Q}^{k-1} \right) 
    \quad \text{s.t.} \quad \EV_{x \sim p_1}[v(x)] = 1.
\end{equation}
To measure deviations from these idealized iterates, let $\hist_k$ be the filtration up to step $k$, and decompose the oracle into its \emph{bias} and \emph{noise} components:
\begin{align}
    \bias_k &\coloneqq \EV \!\left[ \delta \Lfunc(\mathbf{Q}^k) - \delta \Lfunc(\Qmd^k) \,\big|\, \hist_k \right], \\
    \noise_k &\coloneqq \delta \Lfunc(\mathbf{Q}^k) - \delta \Lfunc(\Qmd^k) - \bias_k. 
\end{align}
Here, $\bias_k$ captures systematic error while $\noise_k$ is conditionally mean-zero.  
Under mild assumptions on noise and bias (see \crefrange{asm:precompact}{asm:approximate}), we obtain the following guarantee.  
\begin{tcolorbox}[colframe=white!, top=2pt,left=2pt,right=2pt,bottom=2pt]
\begin{restatable}[Convergence guarantee in the general process-level setting (informal)]{theorem}{trajGeneralCase}
\label{theorem:general_case_convergence}
Suppose the oracle satisfies finite-variance noise and vanishing bias, and let the step-sizes $\{\gamma_k\}$ follow the Robbins--Monro rule 
($\sum_k \gamma_k = \infty$, $\sum_k \gamma_k^2 < \infty$).  
Then the iterates $\{\mathbf{Q}^k\}$ generated by \AlgNameShort satisfy
\begin{equation}
    \mathbf{Q}^k \rightharpoonup \mathbf{Q}^* \quad \text{a.s.},
\end{equation}
where $\mathbf{Q}^* \in \argmax_{\mathbf{Q}} \Lfunc(\mathbf{Q})$.
\end{restatable}
\end{tcolorbox}

\begingroup
  \captionsetup[subfigure]{aboveskip=1.7pt, belowskip=0pt}
\newlength{\imgw}
\setlength{\imgw}{0.25\textwidth}
\newlength{\imgwl}
\setlength{\imgwl}{0.33\textwidth}
\begin{figure*}[t]
    \centering
    % row 1
    \begin{subfigure}{\imgw}
      \centering
      \includegraphics[width=\textwidth]{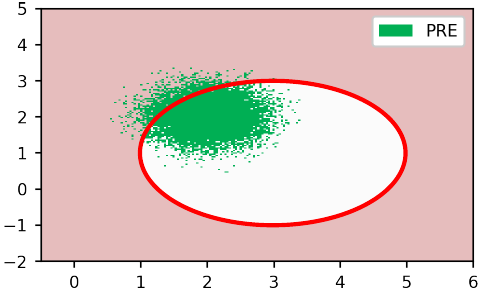}
      \caption{Pre-trained samples}
      \label{fig:toy_top_a}
    \end{subfigure}%
    \begin{subfigure}{\imgw}
      \centering
      \includegraphics[width=\textwidth]{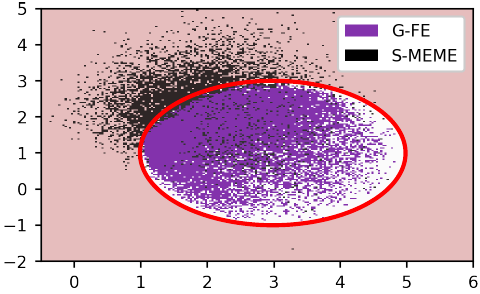}
      \caption{G-\AlgNameShort vs S-MEME}
      \label{fig:toy_top_b}
    \end{subfigure}%
    \begin{subfigure}{\imgw}
      \centering
      \includegraphics[width=\textwidth]{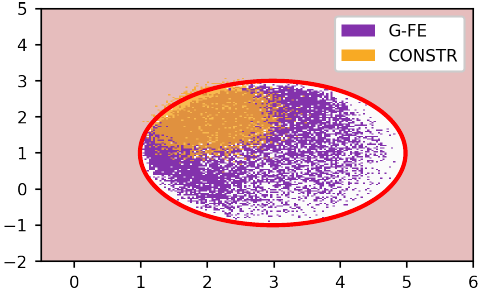}
      \caption{G-FE vs CONSTR}
      \label{fig:toy_top_c}
    \end{subfigure}%
    \begin{subfigure}{\imgw}
      \centering
      \includegraphics[width=\textwidth]{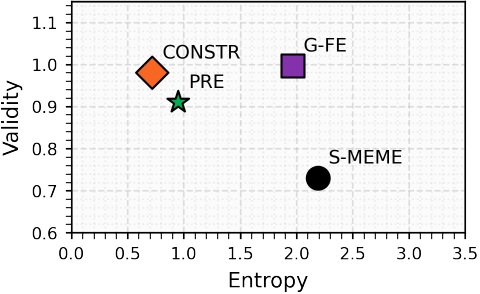}
      \caption{Entropy-Validity}
      \label{fig:toy_top_d}
    \end{subfigure}%
    \\[0.4em]
    % row 2 (repeat)
    \begin{subfigure}{\imgw}
      \centering
      \includegraphics[width=\textwidth]{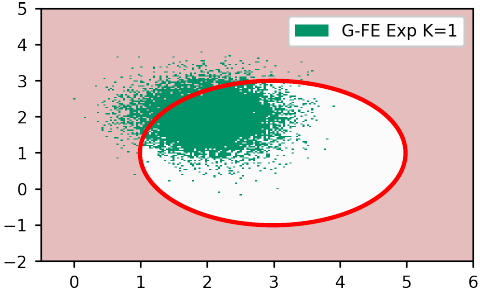}
      \caption{G-FE Expansion $K$=$1$}
      \label{fig:toy_mid_a}
    \end{subfigure}%
    \begin{subfigure}{\imgw}
      \centering
      \includegraphics[width=\textwidth]{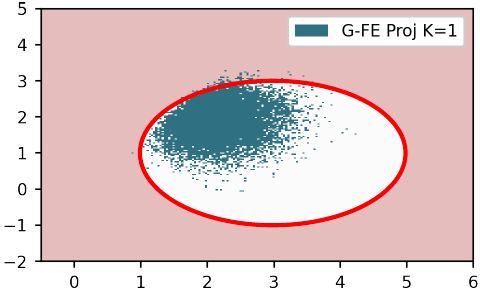}
      \caption{G-FE Projection $K$=$1$}
      \label{fig:toy_mid_b}
    \end{subfigure}%
    \begin{subfigure}{\imgw}
      \centering
      \includegraphics[width=\textwidth]{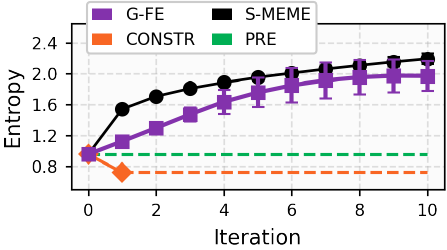}
      \caption{Entropy evaluation}
      \label{fig:toy_mid_c}
    \end{subfigure}%
    \begin{subfigure}{\imgw}
      \centering
      \includegraphics[width=\textwidth]{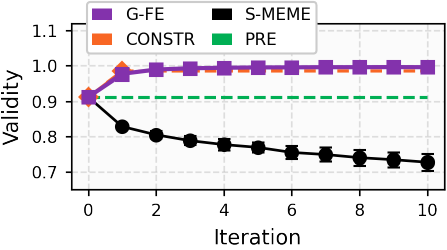}
      \caption{Validity evaluation}
      \label{fig:toy_mid_d}
    \end{subfigure}%
    \caption{\looseness-1 (top) \AlgNameDefGlobal $ $ expands the pre-trained flow model $\pi^{pre}$(\ref{fig:toy_top_a}) into $\pi^*$ (violet, \ref{fig:toy_top_b}), increasing coverage (\ie entropy), while preserving validity (\ie red ellipse interior). Compared with  the unconstrained exploration \AlgNameShortSMEME method, and constrained generation (\AlgNameShortCONSTR), \AlgNameDefGlobal $ $ shows best-of-both-worlds behaviour: achieving near-optimal entropy and validity (Fig. \ref{fig:toy_top_d}). } 
    \label{fig:experiments_fig_1}
\end{figure*}
\endgroup

\section{Experimental Evaluation} 
\label{sec:experiments}

\looseness -1 We analyze the ability of \AlgNameDefGlobal $ $ and \AlgNameDefLocal $ $ to expand flow model densities while preserving validity of generated samples, and compare their performance against recent flow-based exploration methods, namely \AlgNameShortFDC~\citep{santi2025flow}, and \AlgNameShortSMEME~\citep{de2025provable}, as well as a standard constrained generation scheme, denoted by \AlgNameShortCONSTR~\citep[Sec. 8.2][]{uehara2024understanding}. We present experiments on two visually interpretable settings, followed by a molecular design task aiming to increase conformer diversity (more details are provided in Apx. \ref{sec:experimental_detail}).

\mypar{Global Flow Expansion via Strong Verifier}
\looseness -1 We run \AlgNameShortGlobal on a pre-trained model $\pi^{pre}$ to globally expand its density $p_1^{pre}$ over the valid design space (red ellipse in Figs. \ref{fig:toy_top_a}-\ref{fig:toy_top_c}, \ref{fig:toy_mid_a}-\ref{fig:toy_mid_b}). As shown in Fig. \ref{fig:toy_top_b} and \ref{fig:toy_top_c}, \AlgNameShortGlobal (violet) run with $\eta=2$ and $K=10$ , expands into previously uncovered areas (lower right), staying within the valid region. In comparison, \AlgNameShortSMEME (black, Fig. \ref{fig:toy_top_b}) predictably fails to restrict density to the valid region (light red area). Symmetrically, \AlgNameShortCONSTR (see Fig. \ref{fig:toy_top_c}, orange) confines density to the valid space but fails to expand it. Fig. \ref{fig:toy_top_d} shows that \AlgNameShortGlobal explores nearly as much as \AlgNameShortSMEME (\ie $1.97$ vs. $2.17$ entropy), while retaining significantly higher validity: $0.99$ against $0.73$ of \AlgNameShortSMEME. Remarkably, \AlgNameShortGlobal preserves the same degree of validity of \AlgNameShortCONSTR while exploring significantly more ($1.97$ versus $0.72$ entropy). Figs. \ref{fig:toy_mid_a}-\ref{fig:toy_mid_b} show the first expand-then-project steps of \AlgNameShortGlobal and Figs. \ref{fig:toy_mid_c}-\ref{fig:toy_mid_d} show entropy and validity estimates with $95\%$ CI over 5 seeds for \AlgNameShortGlobal and all baselines. In summary, \AlgNameShortGlobal achieves both near-optimal exploration and validity.

\mypar{Local Flow Expansion via Weak Verifier}
We consider a pre-trained flow model $\pi^{pre}$ whose density $p_1^{pre}$ is concentrated in a central high-density region, with low-probability \emph{promising} modes 
\begin{wrapfigure}{r}{0.27\textwidth}
  \centering 
  \includegraphics[width=0.27\textwidth]{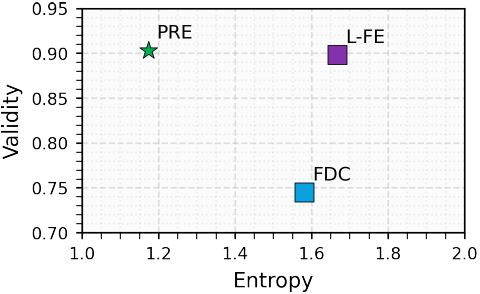}
  \caption{\looseness -1 Entropy-Validity}
  \label{fig:local_pareto_fig} 
\end{wrapfigure}
on either side (see Fig. \ref{fig:toy_local_a}). Crucially, while the two right-most modes are valid, the left one is not.
We fine-tune $\pi^{pre}$ via \AlgNameShortLocal for $K=8$ iterations and $\alpha = 0.99$ to expand its induced density over diverse modes - \ie perform \emph{mode discovery}~\citep{santi2025flow, morshed2025diverseflow}.
As shown in Fig. \ref{fig:toy_local_b} \AlgNameShortFDC, a KL-regularized entropy maximization scheme, predictably increases diversity over plausible modes by redistributing density to the invalid left one. \AlgNameShortLocal, however, leverages a weak verifier (gray circled area in Fig. \ref{fig:toy_local_c}) to prevent allocating more density to that invalid region, and even removes density from that region. 
\ref{fig:toy_local_a}, top). Effectively, \AlgNameShortLocal uses the weak verifier to perform a form of mode selection, \ie filtering out invalid modes during the expansion process. As shown in Fig. \ref{fig:local_pareto_fig}, \AlgNameDefLocal $ $ achieves high entropy (\ie $1.67$ versus $1.17$ and $1.58$ of \AlgNameShortLocal), while preserving high validity, namely $0.89$ compared to $0.74$ of \AlgNameShortFDC, almost fully preserving the prior model's validity of $0.9$.

\begingroup
  \captionsetup[subfigure]{aboveskip=1.7pt, belowskip=0pt}
% \subsection{Experiments on illustrative settings}
\setlength{\imgw}{0.25\textwidth}
\setlength{\imgwl}{0.33\textwidth}
\begin{figure*}[t]
    \centering
    % row 1
    \begin{subfigure}{\imgwl}
      \centering
      \includegraphics[width=\textwidth]{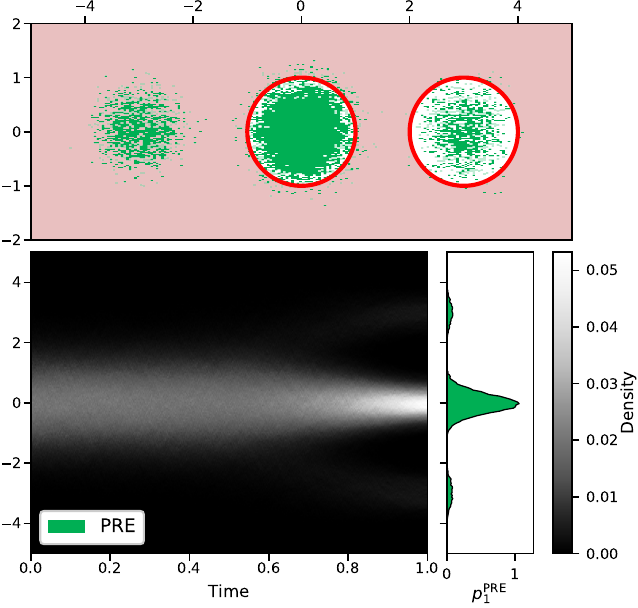}
      \caption{Pre-trained samples}
      \label{fig:toy_local_a}
    \end{subfigure}%
    \begin{subfigure}{\imgwl}
      \centering
      \includegraphics[width=\textwidth]{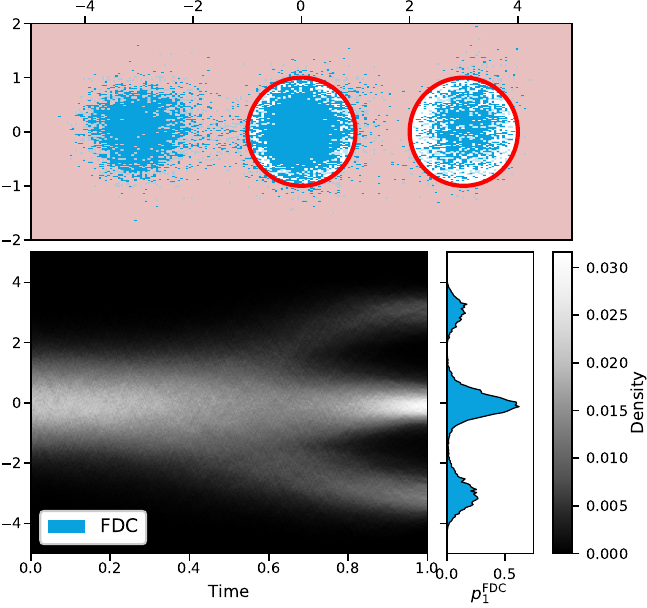}
      \caption{\AlgNameShortFDC samples}
      \label{fig:toy_local_b}
    \end{subfigure}%
    \begin{subfigure}{\imgwl}
      \centering
      \includegraphics[width=\textwidth]{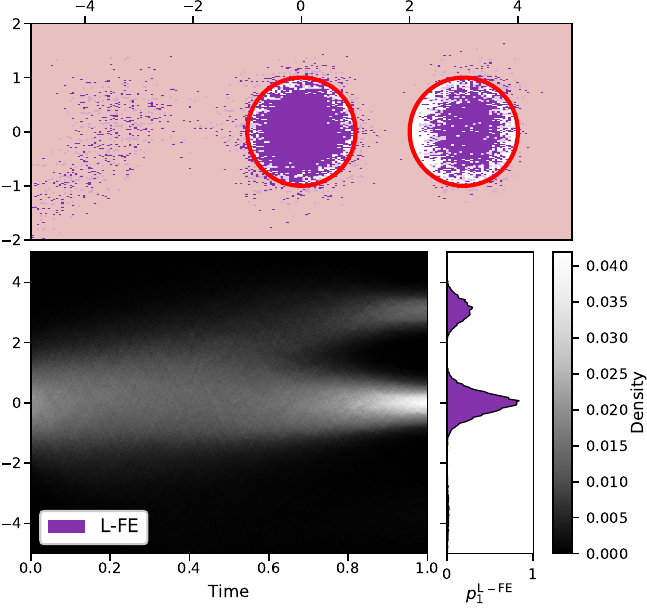}
      \caption{\AlgNameShortLocal samples}
      \label{fig:toy_local_c}
    \end{subfigure}%
    \\[0.4em]
    % row 3 (repeat)
    \begin{subfigure}{\imgw}
      \centering
      \includegraphics[width=\textwidth]{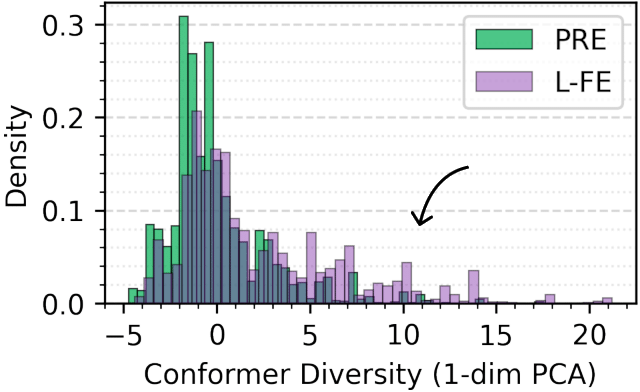}
      \caption{QM9 visual diversity}
      \label{fig:qm9_a}
    \end{subfigure}%
    \begin{subfigure}{\imgw}
      \centering
      \includegraphics[width=\textwidth]{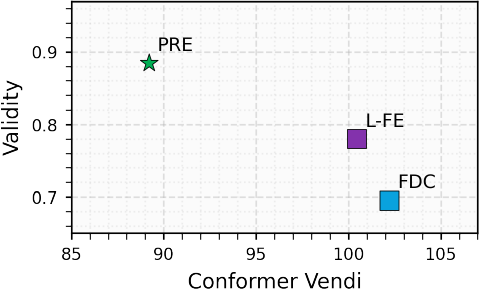}
      \caption{C. Vendi - Validity}
      \label{fig:qm9_b}
    \end{subfigure}%
    \begin{subfigure}{\imgw}
      \centering
      \includegraphics[width=\textwidth]{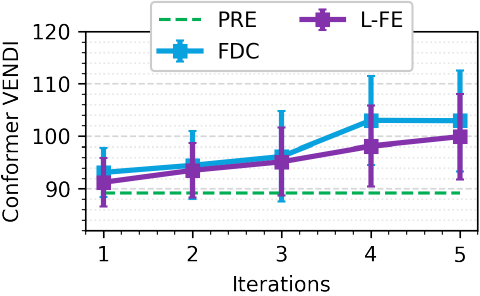}
      \caption{Conformer Vendi}
      \label{fig:qm9_c}
    \end{subfigure}%
    \begin{subfigure}{\imgw}
      \centering
      \includegraphics[width=\textwidth]{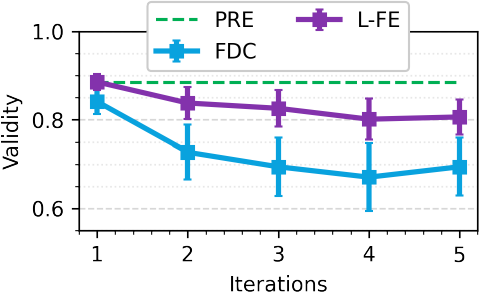}
      \caption{Validity}
      \label{fig:qm9_d}
    \end{subfigure}%
    \caption{\looseness-1 (top) \AlgNameShortLocal (yellow, \ref{fig:toy_local_c}) expands the pre-trained flow model $\pi^{pre}$ (green, \ref{fig:toy_local_a}) over promising yet verifier-filtered modes, while \AlgNameShortFDC (blue, \ref{fig:toy_local_b}) expands $\pi^{pre}$ over all plausible modes leading to increased density in invalid regions (left mode in Fig. \ref{fig:toy_local_b}). (bottom) \AlgNameShort increases visual (\ref{fig:qm9_a}), and quantitative diversity (\ref{fig:qm9_c}), while preserving higher validity than \AlgNameShortFDC (\ref{fig:qm9_b}-\ref{fig:qm9_d})} 
    \label{fig:experiments_fig_2}
\end{figure*}
\endgroup

\mypar{\AlgNameShortLocal increases molecular conformer diversity for de-novo design on QM9} 
\looseness -1 In this experiment, we aim to increase diversity of molecular conformers in a molecular design task.
We run \AlgNameShort on FlowMol CTMC~\citep{dunn2024mixed} pre-trained on QM9 dataset~\citep{ramakrishnan2014quantum}. Our weak verifier is a filter excluding molecules for which any two atoms are closer than 0.975 Ångstroms (Å), and validity is evaluated via RDKit~\citep{rdkit} sanitization paired with the aforementioned check. We evaluate diversity of molecular conformers by a \emph{conformer} VENDI ~\citep{friedman2022vendi} metric (see Apx. \ref{sec:vendi}) capturing diversity over sampled conformers via their fingerprints.
\AlgNameShortLocal, run for $K=5$ iterations and $\alpha = 9$, quantitatively increases diversity compared to the pretrained model (Fig \ref{fig:qm9_b}, VENDI of 100 vs 89). This is visually shown in Fig \ref{fig:qm9_a}, a histogram plot of a $1$-dim PCA projection of molecular fingerprints (see Apx. \ref{sec:experimental_detail} for further details). In particular, \AlgNameShortLocal (violet) expands the pre-trained flow model to explore promising and verifier-certified modes of the pre-trained model density (see Fig. \ref{fig:qm9_a}). Crucially, \AlgNameShortLocal achieves a similar degree of conformer diversity (100 vs 103) to \AlgNameShortFDC, an unconstrained exploration scheme, while preserving significantly higher sample validity, \ie 81\% vs 69\%, as shown in Figs. \ref{fig:qm9_b}, \ref{fig:qm9_c}, and \ref{fig:qm9_d}.

\mypar{\AlgNameShortLocal increases molecular conformer diversity for de-novo design on GEOM-Drugs} 
\looseness -1 In this experiment, we aim to increase the diversity of generated molecular conformers in a molecular design task with drug-like molecules. We run \AlgNameShort on FlowMol CTMC~\citep{dunn2024mixed} pre-trained on GEOM-Drugs  ~\citep{axelrod2022geom}. As in the previous experimental setting, the weak verifier employed is a filter excluding molecules for which any two atoms are closer than 0.975 Ångstroms (Å), and validity is evaluated via the RDKit~\citep{rdkit} sanitization operation paired with the aforementioned check. We evaluate diversity of molecular conformers by a \emph{conformer} VENDI ~\citep{friedman2022vendi} metric (see Apx. \ref{sec:vendi}) capturing diversity over sampled conformers via their fingerprints.
\AlgNameShortLocal, run for $K=3$ iterations, $\alpha = \nicefrac{1}{9}$, and $\eta=5$, achieves higher diversity (529 vs 476) and validity (82\% vs 72\%) than the pre-trained model, as shown in Fig. \ref{fig:fig_new_a}.  Similarly, \AlgNameShortLocal induces higher diversity (529 vs 508) than \AlgNameShortFDC, a recent diffusion-based unconstrained exploration method~\citep{santi2025flow}, while preserving significantly higher sample validity, \ie 82\% vs 66\%, as shown in Fig. \ref{fig:fig_new_a}.

% We report further experimental details (\eg detailed hyperparameters) and an ablation study in Appendix \ref{sec:experimental_detail}.

\looseness - 1 Moreover, within Apx. \ref{sec:ablationstudy}, we report an ablation study for the proposed method parameters. Note that \AlgNameShortLocal performs consistently better than \AlgNameShortExplore. Since \AlgNameShortExplore corresponds to \AlgNameShortLocal with $\eta = 0$, this result illustrates the working mechanism and importance of the projection step (\ie $\eta > 0$). Interestingly, \AlgNameShortGlobal, which is equal to \AlgNameShortLocal with $\alpha_t = 0$, shows performance on par with \AlgNameShortLocal for very conservative parameters ($\gamma=0.0002$, $K=3$), while gradually degrading for less conservative parametrizations. This behaviour is likely due to the implicit KL-regularization between iterates within the mirror descent update step (see Eq. \ref{eq:MD_step}), which implies prior regularization for small $K$.

\mypar{\AlgNameShortExplore achieves higher exploration performance against current methods} 
\looseness -1 We evaluate \AlgNameShortExplore, the unconstrained exploration variant of \AlgNameShortLocal obtained by removing the projection step (see Alg. \ref{alg:algorithm_nse} in Sec. \ref{sec:algorithm} for further details), to perform flow-based design space (unconstrained) exploration. We consider FlowMol CTMC~\citep{dunn2024mixed} pre-trained on GEOM-Drugs dataset ~\citep{axelrod2022geom}, and report in Fig. \ref{fig:fig_new_b} the results for \AlgNameShortExplore with $K=3$ and $\alpha=0$. We observe that \AlgNameShortExplore achieves consistently higher diversity (\ie 519 vs 508) and validity (\ie 74\% vs 66\%) against \AlgNameShortFDC ~\citep{santi2025flow}, a state-of-the-art method for flow-based unconstrained exploration. 

\mypar{\AlgNameShortLocal and \AlgNameShortExplore have computational costs comparable to current exploration schemes} 
\looseness -1 We report in Fig. \ref{fig:fig_new_c} a comparison of computational cost, measured via the method runtime (seconds [s]) of \AlgNameShortLocal and \AlgNameShortExplore compared against \AlgNameShortFDC ~\citep{santi2025flow}. One can notice that although the schemes proposed within this work (\ie \AlgNameShortLocal and \AlgNameShortExplore) perform exploration over the entire flow process noised state space, they do not incur in significantly higher computational cost compared with \AlgNameShortFDC.

\begingroup
  \captionsetup[subfigure]{aboveskip=1.7pt, belowskip=0pt}
% \subsection{Experiments on illustrative settings}
\setlength{\imgw}{0.25\textwidth}
\setlength{\imgwl}{0.33\textwidth}
\begin{figure*}[t]
    \centering
    % row 1
    \begin{subfigure}{\imgw}
      \centering
      \includegraphics[width=\textwidth]{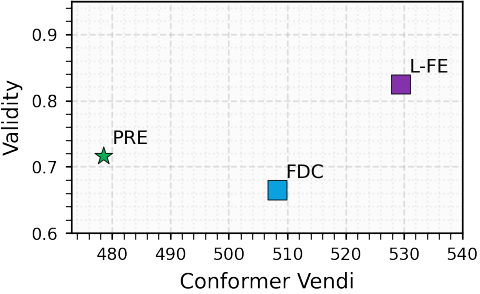}
      \caption{\AlgNameShortLocal on GEOM-Drugs}
      \label{fig:fig_new_a}
    \end{subfigure}%
    \begin{subfigure}{\imgw}
      \centering
      \includegraphics[width=\textwidth]{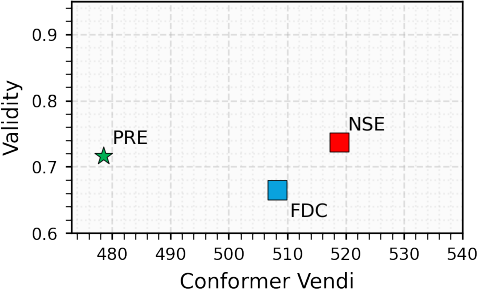}
      \caption{\AlgNameShortExplore vs FDC}
      \label{fig:fig_new_b}
    \end{subfigure}%
    \begin{subfigure}{\imgw}
      \centering
      \includegraphics[width=\textwidth]{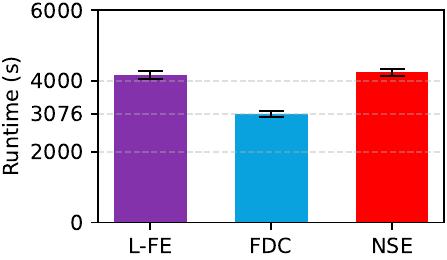}
      \caption{Computational Cost}
      \label{fig:fig_new_c}
    \end{subfigure}%
    \begin{subfigure}{\imgw}
      \centering
      \includegraphics[width=\textwidth]{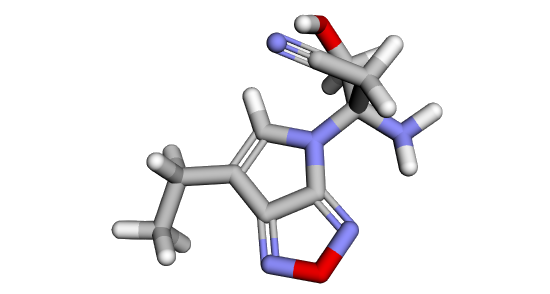}
      \caption{GEOM-drugs molecule}
      \label{fig:fig_new_d}
    \end{subfigure}%
    \\[0.4em]
    \caption{\looseness-1 (\ref{fig:fig_new_a}) \AlgNameShortLocal (violet) achieves higher diversity and validity than \AlgNameShortFDC on GEOM-Drugs. (\ref{fig:fig_new_b}) \AlgNameShortExplore, the unconstrained variant of \AlgNameShortLocal, exhibits superior performance compared with \AlgNameShortFDC, a state-of-the-art diffusion-based unconstrained exploration method. (\ref{fig:fig_new_c}) Computational cost comparison of \AlgNameShortLocal, \AlgNameShortFDC, and \AlgNameShortExplore. (\ref{fig:fig_new_d}) Representative drug-like molecule generated via \AlgNameShortLocal.} 
    \label{fig:experiments_fig_3}
\end{figure*}
\endgroup

\section{Related Work} 
\label{sec:related_works}

\paragraph{Diffusion and flow based design space exploration}
\looseness -1 Recent works introduced methods for flow based design space exploration via maximization of entropy functionals~\citep{de2025provable, santi2025flow} or approximations~\citep{celik2025dime}. While these methods explore by leveraging information from a prior model, \AlgNameShort directs exploration either $(i)$ exclusively via a verifier (\ie global expansion, see \ref{eq:global_flow_expansion_problem}), or $(ii)$ combining verifier information with prior validity cues (\ie local expansion, see \ref{eq:local_flow_expansion_problem}). Moreover, while current schemes explore only the last time-step state space, we lift the exploration task to the entire flow process, providing a principled solution to the score divergence problem mentioned in Sec. \ref{sec:algorithm}.

\mypar{Maximum state entropy exploration}
\looseness -1 Maximum state entropy exploration, introduced by \citet{hazan2019maxent}, tackles the pure-exploration problem of maximizing the entropy of the state distribution induced by a policy over a dynamical system's state space~\citep[\eg][]{lee2019efficient, mutti2021task, guo2021geometric, de2024geometric}. The flow expansion problems (Eq. \ref{eq:global_flow_expansion_problem} and \ref{eq:local_flow_expansion_problem}) are closely related, with $p^\pi_1$ representing the state distribution induced by policy $\pi$ over a subset of the flow process state space (\ie for time-step $t=1$). Recent studies have tackled maximum entropy exploration with finite sample budgets~\citep[\eg][]{de2024global, prajapat2023submodular, mutti2023convex, mutti2022importance, mutti2022challenging}, which could be relevant for future work, \eg  design space exploration under a limited samples constraint.

\mypar{Sample diversity in diffusion models generation}
\looseness -1 A well-known limitation of flow-based generation is limited sample diversity. This problem has been recently addressed by numerous studies~\citep[\eg][]{corso2023particle, um2023don, kirchhof2024sparse, sadat2024cadsunleashingdiversitydiffusion, um2025self, klarner2024context}.
Crucially, such methods are complementary to ours. In fact, they can be applied to promote diverse sampling from the expanded model produced by \AlgNameShort. In particular, whereas these works aim to maximize diversity of a fixed diffusion model or flow model, we aim to sequentially fine-tune a pre-trained flow model so that its induced density is permanently expanded over the valid design space.  Moreover, our formulations (Eq. \ref{eq:global_flow_expansion_problem}, \ref{eq:local_flow_expansion_problem}) and \AlgNameShort scheme increase diversity while integrating validity signal from a chosen verifier.

\section{Conclusion} 
\label{sec:conclusions}
\looseness -1 This work tackles the fundamental challenge of leveraging a verifier (\eg an atomic bonds checker), to expand a pre-trained model's density beyond regions of high data availability, while preserving validity of the generated samples. To this end, we introduce notions of \emph{strong} and \emph{weak} verifiers and cast \emph{global} and \emph{local flow expansion} as probability-space optimization problems. We present \AlgNameDef, a scalable mirror-descent scheme that \emph{provably} solves both problems via verifier-constrained entropy maximization over the flow process noised state space. We provide a thorough analysis showing convergence guarantees for \AlgNameShort under idealized and general assumptions by employing recent mirror-flow theory. Ultimately, we empirically evaluate our method on both illustrative settings, and a molecular design task showcasing the ability of \AlgNameShort to increase molecular conformer diversity while preserving better levels of validity than current flow and diffusion-based exploration methods.

\section*{Acknowledgements}
This publication was supported by the ETH AI Center doctoral fellowship to Riccardo De Santi. The project has received funding from the Swiss
National Science Foundation under NCCR Catalysis grant number 180544 and NCCR Automation grant agreement 51NF40 180545. 

\bibliography{biblio}
\bibliographystyle{iclr2026_conference}

\newpage
\appendix
\section{Appendix}
\tableofcontents
\newpage
\addtocontents{toc}{\protect\setcounter{tocdepth}{2}}

\section{Derivation of Gradients of First Variation}
\label{sec:app_gradients}
In this section we present derivations of the results in \eqref{eq:running-cost} relating the gradient of the first variation of the trajectory rewards to the score function. We derive the result for \AlgNameShortLocal, the result for \AlgNameShortGlobal follows as a subcase.

First, recall the trajectory rewards for \AlgNameShortLocal:

\begin{align}
    \nabla_{x_t}\delta \mathcal{G}_t(p^\pi_t) &= \nabla_{x_t}\delta \big(\mathcal{H}(p^\pi_t) - \alpha_t \mathcal{D}_\textrm{KL}(p^\pi_t || p^\textrm{pre}_t)\big) \\
    &= \nabla_{x_t}\delta \mathcal{H}(p^\pi_t) - \alpha_t \nabla_{x_t}\delta\mathcal{D}_\textrm{KL}(p^\pi_t || p^\textrm{pre}_t)\, . & \text{(by linearity)}
\end{align}

Thus it suffices to show derivations for $\nabla_{x_t}\delta \mathcal{H}(p^\pi_t)$ and $\nabla_{x_t}\delta\mathcal{D}_\textrm{KL}(p^\pi_t || p^\textrm{pre}_t)$. Starting with the entropy functional, recalling its definition as $\mathcal{H}(p^\pi_t) = - \int_0^1p^\pi_t(x) \log p^\pi_t(x) dx$ we have:

\begin{align}
    \nabla_{x_t} \delta \mathcal{H}(p^\pi_t) &= \nabla_{x_t} (1 - \log p^\pi_t) \\
    &= -\nabla_{x_t} \log p^\pi_t \\
    &= -s^\pi_t
\end{align}

Similarly for the second term:

\begin{align}
    \nabla_{x_t} \delta \mathcal{D}_{\textrm{KL}}(p^\pi_t || p^\textrm{pre}_t) &= \nabla_{x_t} \int p^\pi_t \log p^\pi_t - \\
    &= \nabla_{x_t}(\log p^\pi_t - 1 - \log p^\textrm{pre}_t) \\
    &= s^\pi_t - s^\textrm{pre}_t
\end{align}
\newpage

\section{Proof of Proposition \ref{prop:expandthenproject}}
\label{sec:proof_proposition}
In this section we show that the optimization problem in \ref{eq:MD_step} can be decomposed into an unconstrained expansion step followed by a projection into the constrained set. We start by defining the following processes:

\begin{equation} \label{eq:constrained_exp}
    \mathbf{Q}^k \in \argmax_{\mathbf{Q} : p_0 = p^{k - 1}_0 } \langle \delta \mathcal{L}(\mathbf{Q}^{k - 1}), \mathbf{Q}\rangle - \frac{1}{\gamma_k} \mathcal{D}_{\textrm{KL}}(\mathbf{Q} || \mathbf{Q}^{k - 1}) \quad \textrm{s.t.} \quad \EV_{x \sim q_1} \left[v(x)\right] 
\end{equation}

\begin{equation}\label{eq:formal_expansion}
    \Tilde{\mathbf{Q}}^k \in \argmax_{\mathbf{Q} : p_0 = p^{k - 1}_0 } \langle \delta \mathcal{L}(\mathbf{Q}^{k - 1}), \mathbf{Q}\rangle - \frac{1}{\gamma_k} \mathcal{D}_{\textrm{KL}}(\mathbf{Q} || \mathbf{Q}^{k - 1})
\end{equation}

\begin{equation}\label{eq:formal_projection}
    \Bar{\mathbf{Q}}^k \in \argmin_{\mathbf{Q} : p_0 = p^{k - 1}_0} \mathcal{D}_{\textrm{KL}}(\mathbf{Q} || \Tilde{\mathbf{Q}}^k) \; \textrm{s.t.} \; \EV_{x \sim p_1} [v(x)]= 1
\end{equation}

\begin{equation}\label{eq:true_projection}
    \hat{\mathbf{Q}}^k \in \argmax_{\mathbf{Q} : p_0 = p^{k - 1}_0 } \EV_{x \sim p_1}\left[\log v(x)\right] - \frac{1}{\gamma_k} \mathcal{D}_{\textrm{KL}}(\mathbf{Q} || \mathbf{Q}^{k - 1})
\end{equation}

letting $p^k_t$, $\Tilde{p}^k_t$, $\Bar{p}^k_t$, $\hat{p}^k_t$ refer to their respective marginal densities a time $t$. Note that $\Tilde{\mathbf{Q}}^k$ is the output of the projection step in \ref{eq:main_expansion}, and that $\hat{\mathbf{Q}}^k$ is the output of the projection step in \ref{eq:main_proj}. The following Lemma asserts that solving the optimization problem in \eqref{eq:MD_step} is equivalent to solving the expansion step of \ref{eq:formal_expansion} followed by the formal information projection step of \ref{eq:formal_projection}.

\begin{lemma}
    Let $\mathbf{Q}^{k - 1}$ be the process associated with the previous iterate $\pi^{k - 1}$, and let $\Tilde{\mathbf{Q}}^k$ and $\Bar{\mathbf{Q}}^k$ be defined as above. Then $\Bar{\mathbf{Q}}^k = \mathbf{Q}^k$.
\end{lemma}

\begin{proof}
    First, note that the processes $\mathbf{Q}^{k - 1}$ and $\Tilde{\mathbf{Q}}^k_t$ satisfy the following relationship (see e.g. \cite{domingo2024adjoint} equation 22):

    \begin{equation}
        \log \frac{d\Tilde{\mathbf{Q}}^k}{d\mathbf{Q}^{k -1}}(X) = \gamma_k \delta \mathcal{L}(\mathbf{Q}^{k - 1})(X)  + \textrm{const}\, .
    \end{equation}

    which implies the following equality for an arbitrary process $q$ (taking the expectation and rearranging):

    \begin{equation}
        \langle \delta \mathcal{L}(\mathbf{Q}^{k - 1}), \mathbf{Q}\rangle  - \gamma_k D_{\textrm{KL}}(\mathbf{Q} || \mathbf{Q}^{k - 1})  = \gamma_k D_{\textrm{KL}}( \mathbf{Q} || \Tilde{\mathbf{Q}}^k) - \textrm{const} \, .
    \end{equation}

    Therefore the equation below holds for any arbitrary set of processes $A$:

    \begin{equation}
    \argmax_{\mathbf{Q}\in A\\} \langle \delta \mathcal{L}(\mathbf{Q}^{k -1}), q\rangle - \frac{1}{\gamma_k} D_{\textrm{KL}}(\mathbf{Q} || \mathbf{Q}^{k - 1}) = \argmin_{q\in A} D_{\textrm{KL}}(\mathbf{Q} || \Tilde{\mathbf{Q}}^k)
    \end{equation}

    and thus also holds for the set $A = \{\mathbf{Q}  \; \textrm{s.t.} \; p_0 = p_0 ^{k - 1} \,\textrm{and} \, \EV_{x \sim p_1}\left[v(x)\right] = 1\}$: the set of feasible solutions to \ref{eq:MD_step}.

\end{proof}

Finally, the following Lemma reformulates the information projection step in \ref{eq:formal_projection} as the fine-tuning objective in \ref{eq:formal_expansion}:

\begin{lemma}
    Let $\hat{\mathbf{Q}}^k$ and $\Bar{\mathbf{Q}}^k$ be defined as above. Then $\hat{\mathbf{Q}}^k = \Bar{\mathbf{Q}}^k$
\end{lemma}

\begin{proof}
    Recall the definition of $\hat{\mathbf{Q}}^k$:

    \begin{equation}
    \hat{\mathbf{Q}}^k \in \argmax_{\mathbf{Q} : p_0 = p^{k - 1}_0 } \EV_{x \sim p_1}\left[\log v(x)\right] - \frac{1}{\gamma_k} \mathcal{D}_{\textrm{KL}}(\mathbf{Q} || \mathbf{Q}^{k - 1})
\end{equation}

and note that the expectation in the first term is finite only if $v(x) \neq 0, \; p_1 - \textrm{a.s.}$, in which case it vanishes. Thus the maximizer must belong to the set $\big\{\mathbf{Q}: p_0 = p^{k -1 }_0, \; \E_{x\sim p_1}[v(x)] = 1 \big\}$ effectively turning the first term into a constraint.
\end{proof}

\newpage 

\section{Proof for Theorem \ref{theorem:convex_case_convergence}}
\label{sec:app_theory1}
\trajConvexCase*
\newcommand{\Qbf}{\mathbf{Q}}

\renewcommand{\G}{(-\Lfunc)}
\begin{proof}

Fix an initial reference measure $\bar{\mathbf{Q}} \coloneqq \mathbf{Q}^0$, and define the function
\begin{equation}
\label{eq:relative_properties}
\Q(\mathbf{Q}) \coloneqq D_{\mathrm{KL}}\!\left(\mathbf{Q}\,\big\|\, \bar{\mathbf{Q}} \right),
\end{equation}
which measures the Kullback–Leibler divergence of $\mathbf{Q}$ from this reference. This choice of $\Q$ will serve as the \emph{reference function} in the framework of \emph{mirror descent with relative smoothness}~\citep{bauschke2017descent,lu2018relatively}. The key point is that the objective $\Lfunc$ in \cref{eq:noised_flow_expansion_problem} is not necessarily smooth in the classical sense, but it is \emph{$\lambda^\star$-smooth relative to $\Q$}.  

To formalize this, let $D_\Q(\Qbf, \Qbf')$ denote the \emph{Bregman divergence} generated by $\Q$. By definition,
\[
D_\Q(\Qbf, \Qbf') = \Q(\Qbf) - \Q(\Qbf') - \langle \delta \Q(\Qbf'), \Qbf - \Qbf' \rangle.
\]
A direct computation shows that when $\Q$ is the KL divergence from a fixed reference measure, the Bregman divergence reduces exactly to another KL divergence:
\[
D_\Q(\Qbf, \Qbf') = D_{\mathrm{KL}}\!\left(\Qbf \,\big\|\, \Qbf'\right).
\]
This equivalence will allow us to leverage classical properties of relative entropy in the convergence analysis.

\medskip
\noindent
Next, consider the mirror descent iterates $\{\Qbf^k\}$ for minimizing $\G$\footnote{We adopt the standard convention of convex \emph{minimization} rather than concave maximization, which explains the negative sign in the formulation.}. By the definition of relative smoothness, we have
\begin{align}
\label{eq:md_smoothness_step}
\G(\Qbf^k) 
&\leq \G(\Qbf^{k-1}) + \langle \delta \G(\Qbf^{k-1}), \Qbf^k - \Qbf^{k-1} \rangle + \lambda^\star D_\Q(\Qbf^k, \Qbf^{k-1}).
\end{align}
Here, the first inequality follows directly from the \emph{$\lambda^\star$-smoothness of $\G$ relative to $\Q$}, as defined in \cref{eq:relative_properties}. Intuitively, this is a generalization of the standard quadratic upper bound used in classical smooth optimization, but with the Bregman divergence replacing the squared Euclidean norm.  

We can refine this bound further by applying the \emph{three-point inequality} of the Bregman divergence~\citep[Lemma 3.1]{lu2018relatively}. Let us define a linearized function 
\[
\phi(\Qbf) \coloneqq \frac{1}{\lambda^\star} \langle \delta \G(\Qbf^{k-1}), \Qbf - \Qbf^{k-1} \rangle,
\] 
and let $z = \Qbf^{k-1}$, $z^+ = \Qbf^k$. Then the three-point identity gives
\begin{align}
\label{eq:three_point}
\langle \delta \G(\Qbf^{k-1}), \Qbf^k - \Qbf^{k-1} \rangle
&\leq \langle \delta \G(\Qbf^{k-1}), \mu - \Qbf^{k-1} \rangle + \lambda^\star D_\Q(\mu, \Qbf^{k-1}) - \lambda^\star D_\Q(\mu, \Qbf^k),
\end{align}
for any reference point $\mu$. Combining \cref{eq:md_smoothness_step} and \cref{eq:three_point} yields
\begin{align}
\G(\Qbf^k) 
&\leq \G(\Qbf^{k-1}) + \langle \delta \G(\Qbf^{k-1}), \mu - \Qbf^{k-1} \rangle + \lambda^\star D_\Q(\mu, \Qbf^{k-1}) - \lambda^\star D_\Q(\mu, \Qbf^k).
\end{align}

\medskip
\noindent
Finally, we can telescope this inequality over $k=1,\dots,K$. Using the monotonicity of $\G(\Qbf^k)$ along the iterates and the non-negativity of the Bregman divergence $D_\Q$, we obtain~\citep{lu2018relatively}:
\begin{align}
\sum_{k=1}^K \left( \G(\Qbf^k) - \G(\mu) \right)
&\leq \lambda^\star D_\Q(\mu, \Qbf^0) - \lambda^\star D_\Q(\mu, \Qbf^K) 
\leq \lambda^\star D_\Q(\mu, \Qbf^0),
\end{align}
for any $\Qbf$. Dividing both sides by $K$ and rearranging gives a simple \emph{ergodic convergence rate}:
\begin{equation}
\G(\Qbf^K) - \G(\Qbf) \leq \frac{\lambda^\star D_\Q(\Qbf, \Qbf^0)}{K}, 
\end{equation}
which shows that the iterates converge at an $O(1/K)$ rate in terms of the relative entropy.
\end{proof}
\newpage

\section{Proof for Theorem \ref{theorem:general_case_convergence}}
\label{sec:app_theory2}
\newtheorem{informalassumption}{Informal Assumption} 

To establish our main convergence theorem, we impose a few auxiliary assumptions that are widely used in the analysis of stochastic approximation and gradient flows. These assumptions are mild and typically satisfied in practical applications.

\begin{assumption}[Precompactness of Dual Iterates]
\label{asm:precompact}
The sequence of dual variables $\{\delta \Q(\Qbf^k)\}_k$ is precompact in the $L_\infty$ topology.  
\end{assumption}

\noindent
Precompactness ensures that the interpolated trajectories of the dual iterates remain within a bounded region in function space. This property is crucial for applying convergence results based on asymptotic pseudotrajectories, and similar precompactness assumptions have appeared in the literature on stochastic approximation and continuous-time interpolations of discrete dynamics \citep{benaim2006dynamics,hsieh2019finding,mertikopoulos2024unified}.

In our finite-dimensional parameter space, (E.1) essentially requires that the sequence of iterates produced by the solver remains in a bounded set. This is a very mild requirement: it is satisfied as soon as the solver does not diverge numerically (e.g., no exploding parameters or NaNs), which is exactly what we observe in all our experiments. Moreover, standard practices such as bounded initialization, weight decay, and gradient clipping can be viewed as explicit mechanisms that enforce this boundedness.

\begin{assumption}[Noise and Bias Control]
\label{asm:approximate}
The stochastic approximations in the updates satisfy, almost surely, the following conditions:
\begin{align}
   &\|\bias_k\|_\infty \to 0, \\
   &\sum_{k} \EV\!\left[\step_k^2 \big(\|\bias_k\|_\infty^2 + \|\noise_k\|_\infty^2\big)\right] < \infty, \\
   &\sum_{k} \step_k \|\bias_k\|_\infty < \infty.
   \label{eq:bias-step}
\end{align}
\end{assumption}

\noindent
These conditions are standard in the Robbins–Monro framework \citep{robbins1951stochastic,benaim2006dynamics,hsieh2019finding}. They guarantee that the bias of the stochastic updates vanishes asymptotically, and that the cumulative effect of the noise remains controlled. Together, they ensure that the stochastic perturbations do not prevent convergence of the iterates to the optima of the target objective.

With these assumptions in place, we are ready to restate the main result and present its proof.

\begin{tcolorbox}[colframe=white!, top=2pt,left=2pt,right=2pt,bottom=2pt]
\begin{restatable}[Convergence guarantee in the general trajectory setting (rigorous)]{theorem}{trajGeneralCase_rigorous}
\label{theorem:general_case_convergence_girorous}
Suppose the oracle satisfies \crefrange{asm:precompact}{asm:approximate}, and let the step-sizes $\{\gamma_k\}$ follow the Robbins--Monro rule 
($\sum_k \gamma_k = \infty$, $\sum_k \gamma_k^2 < \infty$).  
Then the iterates $\{\mathbf{Q}^k\}$ generated by \AlgNameShort satisfy
\begin{equation}
    \mathbf{Q}^k \rightharpoonup \mathbf{Q}^* \quad \text{a.s.},
\end{equation}
where $\mathbf{Q}^* \in \argmax_{\mathbf{Q}} \Lfunc(\mathbf{Q})$.
\end{restatable}
\end{tcolorbox}

\renewcommand{\dual}{\mathbf{h}}
\renewcommand{\entropy}{\mathcal{Q}}

\begin{proof}

As in the proof of \cref{theorem:convex_case_convergence}, fix an initial reference measure 
\[
\bar{\mathbf{Q}} \coloneqq \mathbf{Q}^0,
\]
and define the relative entropy functional
\begin{equation}
\label{eq:relative_properties_2}
\Q(\mathbf{Q}) \coloneqq D_{\mathrm{KL}}\!\left(\mathbf{Q}\,\big\|\, {\mathbf{\bar{Q}}} \right).
\end{equation}
Correspondingly, we introduce the initial dual variable
\[
\dual_0 \coloneqq \delta \entropy(\mathbf{Q}^0) 
= -\log \frac{\mathrm{d} \mathbf{Q}^0}{\mathrm{d} {\mathbf{\bar{Q}}}},
\]
where $\frac{\mathrm{d} \mathbf{Q}}{\mathrm{d}{\mathbf{\bar{Q}}}}$ denotes the Radon–Nikodym derivative of $\mathbf{Q}$ with respect to $\bar{\mathbf{Q}}$. This dual representation encodes the convex geometry of the problem.

\paragraph{Continuous-time mirror flow.}  
We now consider the continuous-time mirror flow dynamics
\begin{equation}\label{eq:MF}
\tag{MF}
    \begin{cases}
    \dot{\dual}_t = \delta \G(\Qbf_t), \\
    \Qbf_t = \delta \entropy^\star(\dual_t),
    \end{cases}
\end{equation}
where $\entropy^\star$ denotes the Fenchel conjugate of the relative entropy functional.  
Explicitly, we recall that
\[
\entropy^\star(\dual) 
= \log_{\mathbf{\bar{Q}}} \mathbb{E}\!\left[ e^{\dual} \right],
\]
which follows from the variational characterization of the Kullback–Leibler divergence \citep{hsieh2019finding,hiriart2004fundamentals}.

\paragraph{Discrete-to-continuous interpolation.}  
To connect the discrete algorithm with the flow \eqref{eq:MF}, we introduce an interpolation of the iterates.  
Define the linearly interpolated process $\apt{t}$ by
\begin{equation}
\tag{Int}
\label{eq:interpolation}
\apt{t} = \curr + \frac{t - \curr[\efftime]}{\next[\efftime] - \curr[\efftime]} \big(\next - \curr\big), 
\quad 
\curr = \delta \entropy(\mathbf{Q}^k), \quad 
\curr[\efftime] = \sum_{r=0}^k \alpha_r,
\end{equation}
where $\alpha_r$ are the step sizes.  
This construction yields a continuous-time trajectory $\{\apt{t}\}_{t \geq 0}$ that faithfully tracks the discrete iterates in the limit of vanishing step sizes.

\paragraph{Asymptotic pseudotrajectories.}  
We recall the notion of an asymptotic pseudotrajectory (APT), which provides the precise mathematical bridge between discrete stochastic processes and deterministic flows.

Let $\flowmap$ denote the flow map associated with \eqref{eq:MF}; that is, $\flowmap_h(\mathbf{f})$ is the solution of \eqref{eq:MF} at time $h$ when initialized at $\mathbf{f}$.  

\begin{definition}[Asymptotic Pseudotrajectory (APT)]
\label{def:APT}
A trajectory $\apt{t}$ is called an \ac{APT} of \eqref{eq:MF} if, for every finite horizon $T>0$,
\[
\lim_{t\to\infty} \sup_{0 \leq h \leq T} 
\|\apt{t+h} - \flowmap_h(\apt{t})\|_\infty = 0.
\]
\end{definition}

Intuitively, this condition requires that the interpolated sequence asymptotically shadows the exact flow on every bounded time interval.

\paragraph{Limit set characterization.}  
The central result of \cite{benaim2006dynamics} asserts that the long-term behavior of an APT is governed by the internally chain transitive (ICT) sets of the limiting flow.

\begin{theorem}[APT Limit Set Theorem {\citep[Thm.\ 4.2]{benaim2006dynamics}}]
\label{thm:apt2ict}
If $\apt{t}$ is a precompact \ac{APT} of \eqref{eq:MF}, then almost surely its limit set lies within the set of \ac{ICT} points of the flow.
\end{theorem}

\paragraph{Reduction of the convergence proof.}  
With these tools, the convergence analysis reduces to verifying two key claims:
\begin{enumerate}
    \item[\textbf{(C1)}] Under \crefrange{asm:precompact}{asm:approximate}, the interpolated sequence $\{\apt{t}\}$ indeed forms a precompact APT of \eqref{eq:MF}.
    \item[\textbf{(C2)}] The set of \ac{ICT} points of the flow \eqref{eq:MF} coincides with the set of stationary points of $\Lfunc$.
\end{enumerate}

\paragraph{Verification of Claim (C1).}  
Precompactness follows directly from Assumption \ref{asm:precompact}, which guarantees uniform tightness of the sequence of measures and hence compactness of their trajectories in the weak topology. In addition, standard arguments from stochastic approximation \citep{hsieh2019finding,benaim2006dynamics,mertikopoulos2024unified} yield the following quantitative estimate: for every finite horizon $T>0$, there exists a constant $C(T)>0$ such that
\[
\sup_{0 \leq h \leq T} 
\|\apt{t+h} - \flowmap_h(\apt{t}) \| 
\;\leq\; 
C(T) \big[ \Delta(t-1, T+1) + b(T) + \gamma(T) \big],
\]
where $\Delta(t-1, T+1)$ denotes the cumulative effect of noise over the interval $[t-1, t+T+1]$, while $b(T)$ and $\gamma(T)$ capture, respectively, the bias and step-size contributions. This bound quantifies the deviation of the interpolated process from the deterministic mirror flow \eqref{eq:MF}.

\medskip
\noindent\textbf{APT approximation.}  
Under the noise and bias conditions of Assumption \ref{asm:approximate}, both perturbations vanish asymptotically:
\[
\lim_{t \to \infty} \Delta(t-1, T+1) \;=\; 
\lim_{t \to \infty} b(T) \;=\; 0,
\]
uniformly over bounded horizons $T$. Consequently, the discrepancy in the above bound vanishes in the limit, and the interpolated process $\apt{t}$ shadows the continuous-time flow arbitrarily well.

\medskip
Altogether, these arguments show that $\apt{t}$ is indeed a precompact asymptotic pseudotrajectory of the mirror flow.

\paragraph{Verification of Claim (C2).}  
The flow \eqref{eq:MF} is precisely the continuous-time mirror descent dynamics associated with $\G$, which is known to be a \emph{gradient flow} in the spherical Hellinger–Kantorovich geometry \citep{mielke2025hellinger}.  
As such, $\G$ acts as a strict Lyapunov function for the system: along any non-stationary trajectory, $\frac{d}{dt}\G(\Qbf_t) < 0$.  
By \cite[Corollary 6.6]{benaim2006dynamics}, every precompact APT converges to the set of stationary points of the Lyapunov function.  
Since the objective function $\Lfunc$ is the relative entropy, and hence strictly convex, its stationary point coincide with its global minimizer.  

\paragraph{Conclusion.}  
Combining (C1) and (C2) with \cref{thm:apt2ict}, we deduce that the interpolated process $\apt{t}$ converges almost surely to the set of minimizers of $\G$, which readily implies that the original sequence $\{\mathbf{Q}^k\}$ inherits the same convergence guarantee.  
\end{proof}

\newpage

\section{Detailed Example of Algorithm Implementation}
\label{sec:alg_implementation}
In this section we provide comprehensive pseudocode of an example implementation for the two \FineTuningSolver subprocedure in Alg. \ref{alg:expand_then_project}. It is implemented using a variation of Adjoint Matching (AM) which is introduced comprehensively in \citet{domingo2024adjoint}, although we provide pseudocode below for completeness. We note that in principle one could substitute for any other linear fine-tuning method.

Before presenting the implementations, we shortly clarify some relevant notation. The algorithm makes explicit use of the interpolant schedules $\kappa_t$ and $\omega_t$ introduced in \eqref{eq:flow_diff_eq}. We note that in flow model literature they are more commonly known as $\alpha_t$ and $\beta_t$. We denote by $u^\textrm{pre}$ the velocity field corresponding to the pre-trained policy $\pi^\textrm{pre}$, and likewise use $u^\textrm{fine}$ for the velocity field corresponding to the fine-tuned policy. In short, \FineTuningSolver first samples trajectories, which are then used to approximate the solution of a surrogate ODE whose marginals are used as regression targets for the control policy (see \citet{domingo2024adjoint} Section 5 for a full discussion). We note that \FineTuningSolver can be used for objectives with and without trajectory rewards, simply by setting trajectory rewards to zero.

\begin{algorithm}[H]
\caption{Adjoint Matching for fine-tuning Flow Matching models (\FineTuningSolver)}
\begin{algorithmic}[1]
\Require $u^{\text{pre}}$: pre-trained FM velocity field, $\{\nabla f_t\}_{t \in [0,1]}$: gradients of trajectory rewards,$\{\lambda_t\}_{t\in[0,1]}$: (optional) trajectory reward weights, $\gamma$: fine-tuning strength 
\State Initialize fine-tuned vector fields: $u^{\text{finetune}} = u^{\text{pre}}$ with parameters $\theta$.
\For{$n \in \{0, \ldots, N-1\}$}
  \State Sample $m$ trajectories $\bm{X} = (X_t)_{t \in \{0,\ldots,1\}}$ with memoryless noise schedule
  $\sigma(t) = \sqrt{2 \kappa_t \!\left(\tfrac{\dot\omega_t}{\omega_t}\kappa_t - \dot\kappa_t\right)}$, e.g.:
  \begin{equation}
    X_{t+h} = X_t + h\Big( 2 u_\theta^{\text{finetune}}(X_t, t) - \tfrac{\dot\omega_t}{\omega_t}X_t \Big) 
      + \sqrt{h}\,\sigma(t)\,\varepsilon_t, 
      \quad \varepsilon_t \sim \mathcal{N}(0,I), \quad X_0 \sim \mathcal{N}(0,I). \tag{51}
  \end{equation}
  \State For each trajectory, solve the \emph{lean adjoint ODE} backwards in time from $t=1$ to $0$, e.g.:
  \begin{equation}
    \tilde{a}_{t-h} = \tilde{a}_t + h\,\tilde{a}_t^\top \nabla_{X_t} 
    \Big( 2 v^{\text{base}}(X_t, t) - \tfrac{\dot\omega_t}{\omega_t}X_t \Big) - h \gamma \lambda_t\nabla_{X_t} f_t(X_t),
    \quad \tilde{a}_1 = \gamma \lambda_1 \nabla_{X_1} f_1(X_1). \tag{52}
  \end{equation}
  \State Note that $X_t$ and $\tilde{a}_t$ should be computed without gradients, i.e.,
  $X_t = \texttt{stopgrad}(X_t), \ \tilde{a}_t = \texttt{stopgrad}(\tilde{a}_t)$.
  \State For each trajectory, compute the following Adjoint Matching objective:
  \begin{equation}
    \mathcal{L}_{\text{Adj-Match}}(\theta) 
      = \sum_{t \in \{0,\ldots,1-h\}} 
      \left\| \tfrac{2}{\sigma(t)} \Big(v_\theta^{\text{finetune}}(X_t, t) - u^{\text{base}}(X_t, t)\Big) 
        + \sigma(t)\,\tilde{a}_t \right\|^2. \tag{53}
  \end{equation}
  \State Compute the gradient $\nabla_\theta \mathcal{L}(\theta)$ and update $\theta$ using favorite gradient descent algorithm.
\EndFor
\label{alg:AM_oracle}
\end{algorithmic}
\textbf{Output:} Fine-tuned vector field $v^{\text{finetune}}$
\end{algorithm}

Crucially, we employ the fine-tuning oracle in Alg. \ref{alg:AM_oracle} also to implement the projection step within Sec. \ref{sec:algorithm}, as indicated within Alg. \ref{alg:expand_then_project}. Moreover, in the case of a non-differentiable (weak or strong) verifier, the projection step can be implemented via a $0$-th order RL-based fine-tuning method, e.g.,~\citep{fan2023dpok, black2023training}, which induce the same closed-form solution as Alg. \ref{alg:AM_oracle}.

\newpage

\section{Experimental Details}
\label{sec:experimental_detail}
\subsection{Illustrative Examples Experimental Details}
Numerical values in all plots shown within Sec. \ref{sec:experiments} are means computed over diverse runs of \AlgNameShort via $5$ different seeds. Error bars correspond to $95\%$ Confidence Intervals. For the following comparisons, we aimed to tune each algorithm parameters so that the method would work well in the specific illustrative example.

\mypar{Pre-trained models} The pre-trained models appearing in Sec. \ref{sec:experiments}, in the context of illustrative examples, are learned on synthetically generated data, via standard learning procedures. In particular, in Sec. \ref{sec:experiments} we always show samples generated by such pre-trained models.  

\mypar{Global Flow Expansion}
\begin{itemize}
    \item For \AlgNameShortGlobal, we use $\lambda_t = 0$ if $t > 1-0.05$, and $\lambda_t = 1.2$ otherwise, $\gamma_k = \frac{1.5}{(1 + 3(k-1))}$, $\eta =2$ and $K=10$. 
    \item For \AlgNameShortCONSTR we employ $\eta = 2$.
    \item For \AlgNameShortSMEME we employ $\gamma_k = \frac{0.345}{(1 + 3(k-1))}$ and $K=10$ and use $s^\pi_1(x) = s^\pi_{1-\epsilon}(x)$ with $\epsilon = 0.02$  as discussed in Sec. \ref{sec:algorithm}.
\end{itemize}

\mypar{Local Flow Expansion}
The models used act on a $2$-dim state $(x_1,x_2)$, of which is shown only the $x_1$ coordinate in the process-level figures reported in Sec. \ref{sec:experiments}. Since we use as oracle AM, which requires differentiable gradient, we consider a binary verifier (shown in Fig. \ref{fig:toy_local_c} in grey), which we smoothen, thus rendering it differentiable and approximate. Notice that differentiability is not required by \AlgNameShort, but is rather an implementation detail due to the specific oracle used (\ie AM~\citep{domingo2024adjoint}, see Sec. \ref{sec:alg_implementation} for further details). In particular, there exist several analogous oracles that do not require function differentiability~\citep[\eg][]{fan2023dpok}.
\begin{itemize}
    \item For \AlgNameShortLocal we employ $K=8$, $\lambda_t = 0$ if $t > 1-0.015$, and proportional with the process variance, \ie $\lambda_t = \sqrt{2(\kappa_t(\tfrac{\dot\omega_t}{\omega_t}\kappa_t - \dot\kappa_t))}$, otherwise; $\gamma_k = 0.3$, $\eta_k = 0.1$.
    \item For \AlgNameShortFDC we use $K=8$, $\gamma_k = 0.06$ and use $s^\pi_1(x) = s^\pi_{1-\epsilon}(x)$ with $\epsilon = 0.02$ as discussed in Sec. \ref{sec:algorithm}.
\end{itemize}

\subsection{Conformer VENDI}\label{sec:vendi}

We begin with a detailed explanation of our diversity metric: conformer VENDI. In general, VENDI \citep{friedman2022vendi} is a diversity metric operating on arbitrary inputs based on a pairwise distance kernel $k: \mathcal{X} \times \mathcal{X} \rightarrow [0, 1]$. For a list of inputs $x_1, \ldots, x_n$ and a symmetric pairwise distance kernel $k$, VENDI is defined as:

\begin{equation}
    VS(x_1, \ldots, x_n) = \exp\big(- \sum_{i = 1}^n \lambda_i \log \lambda_i\big)
\end{equation}

where $\lambda_i$ are eigenvalues of the distance matrix $K$ with $K_{ij} = k(x_i, x_j)$. In our work, we use a molecular fingerprinting method combined with a kernel simply defined as the Euclidean distance between fingerprints, thereby inducing a kernel on molecules themselves. The particular fingerprinting method is defined as the sorted list of pairwise atomic distances. Formally, for a molecule with $N$ atoms at positions $a_1, \ldots a_N$ we first compute the matrix of pairwise distances $A_{ij} = \lVert a_i - a_j\rVert_2, \; 1 \leq i < j\leq N$, which is then sorted as $\Tilde{A}_{i_1j_1} \ldots \Tilde{A}_{i_{N(N-1)/2}j_{N(N-1)/2}}$ yielding the fingerprint $\Tilde{A} \in \mathbb{R}^{N(N-1)/2}$.

\subsection{PCA Projection for Fig. \ref{fig:qm9_a}}

In this section we explain the dimensionality reduction method used to generate the plot in \ref{fig:qm9_a}. We first generated 25000 molecules from both the pre-trained model $\pi^\textrm{pre}$ (yielding $D^\textrm{pre}$ and the fine-tuned model $\pi^K$ (yielding $D^\textrm{finetuned}$), which was computed by the \AlgNameShortLocal algorithm on $\pi^\textrm{pre}$ for $K=5$ iterations with $\alpha = 9$. We then fingerprinted each set of molecules using the method described above, and fit a $1$-dim PCA on the fingerprints for $D^\textrm{pre}$ using \textsc{scikit-learn} \citep{scikit-learn}, which was then used to transform both $D^\textrm{pre}$ and $D^\textrm{finetuned}$ into $1$-dimensional vectors. Fig. \ref{fig:qm9_a} corresponds to a histogram plot of each of the resulting sets of vectors.

\subsection{Validity Computation in Molecular Experiments}

In the context of our experiments on molecules the concept of validity is defined through a pipeline of several checks, defined below. Our validity function passes through each one sequentially, returning an invalid result if any fail, and a valid result only if all checks pass.

\begin{enumerate}
    \item First, we attempt to sanitize each molecule using RDKit's \citep{rdkit} \textsc{Chem.SanitizeMol} function. As an added check, we test if it is possible to convert the molecule to and back from SMILES \citep{Weininger1988} notation.
    \item We then iterate over each atom in the molecule, checking for any implicit hydrogens (our model must generate explicit hydrogens as FlowMol \citep{dunn2024mixed} does) or any radical electrons which would make the molecule invalid.
    \item Finally we perform our weak verifier check, filtering out molecules for which any two atoms are closer than 0.9Å. Details on this weak verifier are explained below.
\end{enumerate}

The final validity check is evaluate the weak verifier on molecules that pass the previous steps. The weak verifier itself is evaluated by first computing the vector of pairwise distances between atoms $\Tilde{A}$ (see discussion in Section \ref{sec:vendi} above), then taking the minimum element $\Tilde{A}_0$ and checking if it is lower than $0.9$Å, in which case a molecule is classified as invalid. Including this check in the validity function guarantees by construction that the weak verifier satisfies Definition \ref{definition:weak_verifier}, since failing the weak verifier check implies failing the validity check as well.

\subsection{Practical Details for Experiments on Molecules}

In this section we discuss the practical choices behind the molecular design experiments discussed in Section \ref{sec:experiments}. We start with a discussion of hyperparameter settings, followed by some implementation techniques adapting the verifier feedback for a first-order solver, and finally discuss hardware and platform used for training.

\subsubsection{Hyperparameter choices}

For our experiments on molecules we use the FlowMol CTMC model from \citet{dunn2024mixed} trained on the QM9 dataset as a pre-trained model. We run each algorithm (\AlgNameShortLocal and \textsc{FDC}) with the following parameters:

\begin{itemize}
    \item $K = 5$ iterations
    \item Regularization strength of $\alpha = 9$
    \item Decreasing stepsize of  $\gamma_k = \frac{\gamma_0}{1 + k}$ with $\gamma_0 = 0.00001$
    \item For the trajectory reward weighting (\AlgNameShortLocal only) we use $\lambda_t = \sigma_t = \sqrt{2\kappa_t \big(\tfrac{\dot\omega}{\omega}\kappa_t - \dot\kappa_t\big)}$, ensuring $\lambda_t \rightarrow 0$ as $t\rightarrow 1$ for stability as discussed at the end of section \ref{sec:algorithm}
    \item For both, we clip the score near the end of the trajectory as $s^\pi_t = s^\pi_{\{\min{t, 1 - \epsilon}\}}$ for $\epsilon = 0.005$. 
    \item We fix the number of atoms in generated molecules (for model training and metric calculation) to $10$, in order to simplify metric calculations.
\end{itemize}

When using Adjoint Matching (AM) \citep{domingo2024adjoint} to implement the subroutines of any algorithm we use $N = 4$ iterations, we sample a batch of $m = 4$ trajectories of length $40$ at each iteration and update the parameters $\theta$ using Adam \cite{kingma2015} with a learning rate of $0.00055$. We note that since FlowMol CTMC is a mixed categorical and continuous flow model, we only use AM to update the parameters corresponding to the continuous outputs of the model, \ie the atom positions.

\subsubsection{Hyperparameter Choices for Experiments on GEOM}\label{sec:hypergeom}

In order to test the performance of our model in a more practical setting, we performed additional experiments using the GEOM dataset. We again use a FlowMol CTMC model from \citet{dunn2024mixed} as a pre-trained model, however this time using checkpoints from training on the GEOM dataset. Note that for simpler hyperparameter search we use the reparametrization discussed in section \ref{sec:reparametrization}. The optimal parameter set for each algorithm (\AlgNameShortLocal, \AlgNameShortGlobal, \AlgNameShortExplore, and \AlgNameShortFDC) is not identical in this setting, therefore we first report the hyperparameters in common before listing the differences for each algorithm below:

\textbf{Common hyperparameters for \AlgNameShortLocal, \AlgNameShortGlobal, \AlgNameShortExplore, \AlgNameShortFDC}:
\begin{itemize}
    \item $K = 3$ iterations
    \item Constant (adjusted) stepsize $\Tilde{\gamma_k} = \Tilde{\gamma_0}$, although the magnitude $\Tilde{\gamma_0}$ differs for each algorithm 
    \item For the trajectory reward weighting (\AlgNameShortLocal, \AlgNameShortGlobal and \AlgNameShortExplore) we use $\lambda_t = \sigma_t = \sqrt{2\kappa_t \big(\tfrac{\dot\omega}{\omega}\kappa_t - \dot\kappa_t\big)}$, ensuring $\lambda_t \rightarrow 0$ as $t\rightarrow 1$ for stability as discussed at the end of section \ref{sec:algorithm}
    \item For all methods, we clip the score near the end of the trajectory as $s^\pi_t = s^\pi_{\min\{t, 1 - \epsilon\}}$ for $\epsilon = 0.005$. 
    \item We fix the number of atoms in generated molecules (for model training and metric calculation) to $30$, in order to simplify metric calculations.
\end{itemize}

\textbf{Per-algorithm hyperparameter variations:}

\begin{itemize}
    \item \AlgNameShortLocal: 
     \begin{itemize}
         \item $\beta = 0.4$
         \item $\Tilde{\gamma}_k = 0.0005$ (constant)
         \item $\eta_k = 5.0$ (constant)
     \end{itemize}
    \item \AlgNameShortExplore:
    \begin{itemize}
        \item $\beta = 0.0$
        \item $\Tilde{\gamma}_k = 0.0002$ (constant)
    \end{itemize}
    \item \AlgNameShortFDC:
    \begin{itemize}
        \item $\beta = 0.9$
        \item $\Tilde{\gamma}_k = 0.0005$ (constant)
    \end{itemize}
\end{itemize}

When using Adjoint Matching (AM) \citep{domingo2024adjoint} to implement the subroutines of any algorithm usign the GEOM FlowMol model we use $N = 4$ iterations, we sample a batch of $m = 1$ trajectories of length $40$ at each iteration and update the parameters $\theta$ using Adam \cite{kingma2015}, with a learning rate of $0.0001$. We note that since FlowMol CTMC is a mixed categorical and continuous flow model, we only use AM to update the parameters corresponding to the continuous outputs of the model, \ie the atom positions.

\subsubsection{Smoothing the Weak Verifier}

Since we use Adjoint Matching for all fine-tuning tasks we need all rewards to be differentiable. While our weak verifier is formally defined as $v(x) = 1 \iff$ $x$ respects the minimum atom separation bound of $0.9$Å, we use the following differentiable approximation using a sigmoid soft indicator function:

\begin{equation}
    v(x) = \frac{1}{N(N-1)/2}\sum_{i = 1}^{N (N-1) /2}\frac{\exp(\Tilde{A}_i - 0.9)}{\exp(\Tilde{A}_i - 0.9) + 1}
\end{equation}

where $\Tilde{A}_i$ are the pairwise atomic distances introduced in Section \ref{sec:vendi} above. This alternative verifier is differentiable and provides gradient feedback everywhere and therefore can be used in Adjoint Matching.

\subsubsection{Ablation Study}
\label{sec:ablationstudy}

We report an ablation study of the key hyperparameters $\alpha$, $\gamma_k$, and $\eta_k$. Notice first that since \AlgNameShortLocal is a generalisation of both \AlgNameShortExplore and \AlgNameShortGlobal, we recover \AlgNameShortExplore by setting $\eta_k = 0$, and recover \AlgNameShortGlobal by instead setting $\alpha = 0$. The following plot shows a comparison of the Pareto fronts for each method, all compared against \AlgNameShortFDC. Furthermore, at the end of this section, we report results of running the projection step alone (effectively setting $\gamma_k = 0$).

\begin{figure}[ht]
    \centering
    \includegraphics[width=0.95\linewidth]{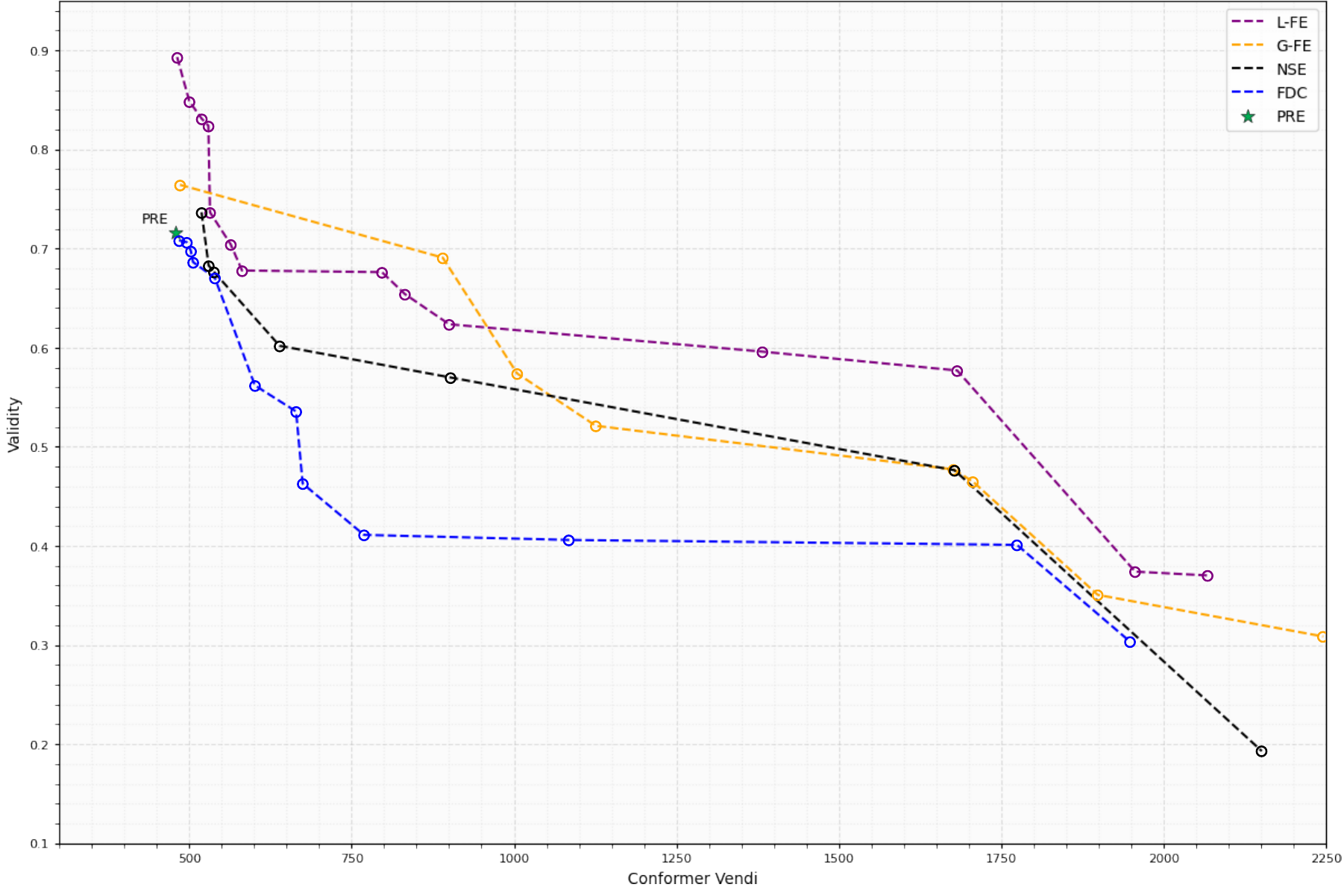}
    \caption{Comparison of different parametrization of \AlgNameShortLocal, \AlgNameShortGlobal, \AlgNameShortExplore, \AlgNameShortFDC, for $K=3$. We compute Conformer Vendi and validity over 5000 samples, and average results over 8 runs for each parameterization with 8 different seeds.}
    \label{fig:ablations}
\end{figure}

In Fig. \ref{fig:ablations} we report all the Pareto dominant points chosen for each method evaluated in the following ranges of hyperparameters (using the reparametrized notation from Appendix \ref{sec:reparametrization}):

\begin{itemize}
    \item $\beta$: $0.0$, $0.1$, $0.2$, $0.3$, $0.4$, $0.5$, $0.6$, $0.7$, $0.8$, $0.9$, $0.95$
    \item $\Tilde{\gamma_k}$ (constant): $0.0001$, $0.0002$, $0.0003$, $0.0004$, $0.0005$, $0.001$
    \item $\eta_k$ (constant): $0.0$, $0.1$, $0.5$, $1.0$, $2.0$, $5.0$
\end{itemize}

For each combination of the hyperparameters above, we run $K=3$ iterations of each method (\AlgNameShortLocal, \AlgNameShortGlobal and \AlgNameShortExplore) and average the results across 8 seeds. Conformer Vendi and validity are always computed for batches of 5000 samples. For all other hyperparameters refer to the previous section \ref{sec:hypergeom} discussing general hyperparameter choices for the GEOM model. To build the plot above, we simply drop all points that are Pareto dominated (i.e. same/greater Conformer Vendi and same/greater validity) by some other point, and plot the remaining points. We also drop all points with Conformer Vendi lower than the pre-trained model. The results shown in Fig. \ref{fig:ablations} clearly demonstrate the ability of each method to trade off validity for higher diversity (measured by Conformer Vendi). However notice that only \AlgNameShortGlobal and \AlgNameShortLocal can significantly increase validity due to their use of verifier information. Overall, notice that both \AlgNameShortLocal and \AlgNameShortGlobal outperform all other methods almost uniformly, and \AlgNameShortExplore uniformly outperforms \AlgNameShortFDC (both unconstrained exploration algorithms). \AlgNameShortLocal is especially effective at retaining or even increasing validity: it significantly dominates all other methods in terms of validity in the range of $475$ to $550$ Conformer Vendi (top left of Fig. \ref{fig:ablations}).

\subsubsection{Tables of results}

We report the numerical results (value and confidence interval)  for each point in Fig. \ref{fig:ablations} in the tables below. For each method we report the mean and confidence interval corresponding to each point from left to right, and report the parameterization ($\beta$, $\Tilde{\gamma}_0$ and $\eta$ if applicable) used to generate the point. We acknowledge the size of the confidence intervals increases drastically in the regime of more exploratory points (higher Conformer Vendi): increasing stability of exploration methods in that regime remains an open problem. Still, we note that \AlgNameShortLocal dominates other methods most significantly when the Conformer Vendi increase is modest (less than $550$ Conformer Vendi), and in that regime the confidence intervals are reasonably concentrated.

\begin{table}[ht]
  \centering
  \caption{Values for \AlgNameShortLocal reported in Fig \ref{fig:ablations} with 95\% confidence intervals}
  \label{tab:ciablations1}
  \begin{tabular}{>{}c >{}c >{}c >{}c >{}c}
    \toprule    
    $\beta$ & $\Tilde{\gamma}_0$ & $\eta_0$ & Mean Validity (95\% CI) & Mean Conformer Vendi (95\% CI) \\
    \midrule
    0.3 & 0.0002 & 5.0 & 0.89 (0.87 $\pm$ 0.91) & 481.35 (460.24 $\pm$ 500.42)\\
    0.6 & 0.0005 & 5.0 & 0.85 (0.79 $\pm$ 0.89) & 500.11 (460.72 $\pm$ 546.91)\\
    0.4 & 0.0005 & 5.0 & 0.83 (0.76 $\pm$ 0.89) & 519.05 (454.45 $\pm$ 585.99)\\
    0.1 & 0.0004 & 5.0 & 0.82 (0.70 $\pm$ 0.90) & 529.47 (465.92 $\pm$ 629.29)\\
    0.2 & 0.0004 & 5.0 & 0.74 (0.53 $\pm$ 0.88) & 531.90 (452.14 $\pm$ 633.92)\\
    0.2 & 0.0005 & 5.0 & 0.70 (0.53 $\pm$ 0.86) & 563.75 (472.19 $\pm$ 683.79)\\
    0.2 & 0.0004 & 2.0 & 0.68 (0.50 $\pm$ 0.83) & 581.26 (467.44 $\pm$ 728.66)\\
    0.3 & 0.0005 & 5.0 & 0.68 (0.46 $\pm$ 0.85) & 796.64 (475.55 $\pm$ 1371.21)\\
    0.2 & 0.0005 & 2.0 & 0.65 (0.43 $\pm$ 0.84) & 831.91 (496.49 $\pm$ 1423.91)\\
    0.1 & 0.0004 & 2.0 & 0.62 (0.39 $\pm$ 0.82) & 900.08 (515.31 $\pm$ 1591.61)\\
    0.1 & 0.0005 & 5.0 & 0.60 (0.33 $\pm$ 0.83) & 1382.46 (516.37 $\pm$ 2493.28)\\
    0.4 & 0.0005 & 0.5 & 0.58 (0.31 $\pm$ 0.81) & 1682.32 (560.86 $\pm$ 3320.25)\\
    0.1 & 0.0005 & 2.0 & 0.37 (0.13 $\pm$ 0.64) & 1956.33 (684.10 $\pm$ 3322.31)\\
    0.1 & 0.0005 & 0.5 & 0.37 (0.12 $\pm$ 0.64) & 2068.02 (866.88 $\pm$ 3366.56)\\

    \bottomrule
  \end{tabular}

  \vspace{1ex}
  \footnotesize{  Values are mean (95\% confidence interval). Confidence intervals computed using bootstrapping and shown to two decimal places.}
\end{table}

\begin{table}[ht]
  \centering
  \caption{ Values for \AlgNameShortGlobal reported in Fig \ref{fig:ablations} with 95\% confidence intervals}
  \label{tab:ciablations2}
  \begin{tabular}{>{}c >{}c >{}c >{}c}
    \toprule    
    $\Tilde{\gamma}_0$ & $\eta_0$ & Mean Validity (95\% CI) & Mean Conformer Vendi (95\% CI) \\
    \midrule
    0.0002 & 0.5 & 0.76 (0.69 $\pm$ 0.83) & 485.65 (442.52 $\pm$ 524.71)\\
    0.0005 & 5.0 & 0.69 (0.46 $\pm$ 0.86) & 890.19 (434.36 $\pm$ 1731.52)\\
    0.0004 & 0.5 & 0.57 (0.33 $\pm$ 0.79) & 1004.79 (531.74 $\pm$ 1790.78)\\
    0.0004 & 1.0 & 0.52 (0.27 $\pm$ 0.77) & 1125.94 (521.23 $\pm$ 2182.77)\\
    0.0004 & 0.1 & 0.48 (0.22 $\pm$ 0.72) & 1676.67 (533.47 $\pm$ 2888.04)\\
    0.0005 & 1.0 & 0.46 (0.22 $\pm$ 0.71) & 1706.53 (664.86 $\pm$ 3100.05)\\
    0.0005 & 2.0 & 0.35 (0.12 $\pm$ 0.60) & 1898.72 (796.05 $\pm$ 3221.32)\\
    0.0005 & 0.5 & 0.31 (0.08 $\pm$ 0.56) & 2245.62 (1135.37 $\pm$ 3500.43)\\
    \bottomrule
  \end{tabular}

  \vspace{1ex}
  \footnotesize{  Values are mean (95\% confidence interval). Confidence intervals computed using bootstrapping and shown to two decimal places.}
\end{table}

\begin{table}[ht]
  \centering
  \caption{ Values for \AlgNameShortExplore reported in Fig \ref{fig:ablations} with 95\% confidence intervals}
  \label{tab:ciablations3}
  \begin{tabular}{>{}c >{}c >{}c >{}c >{}c}
    \toprule    
    $\beta$ & $\Tilde{\gamma}_0$ & Mean Validity (95\% CI) & Mean Conformer Vendi (95\% CI) \\
    \midrule
    0.0 & 0.0002 & 0.74 (0.68 $\pm$ 0.81) & 518.92 (505.70 $\pm$ 531.90)\\
    0.4 & 0.0004 & 0.68 (0.58 $\pm$ 0.77) & 529.85 (484.34 $\pm$ 580.24)\\
    0.3 & 0.0004 & 0.68 (0.55 $\pm$ 0.79) & 537.83 (491.25 $\pm$ 588.98)\\
    0.1 & 0.0004 & 0.60 (0.42 $\pm$ 0.77) & 639.04 (501.23 $\pm$ 857.00)\\
    0.2 & 0.0004 & 0.57 (0.37 $\pm$ 0.74) & 902.23 (534.29 $\pm$ 1559.95)\\
    0.1 & 0.0005 & 0.48 (0.25 $\pm$ 0.69) & 1678.29 (669.60 $\pm$ 3024.30)\\
    0.0 & 0.0005 & 0.19 (0.02 $\pm$ 0.40) & 2151.14 (1256.65 $\pm$ 3162.47)\\

    \bottomrule
  \end{tabular}

  \vspace{1ex}
  \footnotesize{  Values are mean (95\% confidence interval). Confidence intervals computed using bootstrapping and shown to two decimal places.}
\end{table}

\begin{table}[ht]
  \centering
  \caption{ Values for \AlgNameShortFDC reported in Fig \ref{fig:ablations} with 95\% confidence intervals}
  \label{tab:ciablations4}
    \begin{tabular}{>{}c >{}c >{}c >{}c >{}c}
        \toprule    
    $\beta$ & $\Tilde{\gamma}_0$ & Mean Validity (95\% CI) & Mean Conformer Vendi (95\% CI) \\
    \midrule
    0.4 & 0.0001 & 0.71 (0.67 $\pm$ 0.74) & 484.34 (461.47 $\pm$ 506.75)\\
0.5 & 0.0002 & 0.71 (0.63 $\pm$ 0.77) & 496.06 (467.29 $\pm$ 523.97)\\
0.7 & 0.0002 & 0.70 (0.64 $\pm$ 0.75) & 502.49 (483.86 $\pm$ 526.78)\\
0.5 & 0.0001 & 0.69 (0.65 $\pm$ 0.74) & 505.84 (468.66 $\pm$ 538.20)\\
0.2 & 0.0003 & 0.67 (0.58 $\pm$ 0.77) & 539.29 (504.70 $\pm$ 585.65)\\
0.9 & 0.001 & 0.56 (0.49 $\pm$ 0.62) & 601.06 (531.19 $\pm$ 699.07)\\
0.3 & 0.0004 & 0.54 (0.34 $\pm$ 0.72) & 664.45 (531.77 $\pm$ 847.98)\\
0.5 & 0.0005 & 0.46 (0.30 $\pm$ 0.66) & 674.63 (544.65 $\pm$ 824.76)\\
0.4 & 0.0005 & 0.41 (0.25 $\pm$ 0.58) & 768.98 (599.68 $\pm$ 975.95)\\
0.1 & 0.0004 & 0.41 (0.24 $\pm$ 0.56) & 1084.21 (574.88 $\pm$ 2062.43)\\
0.1 & 0.0005 & 0.40 (0.15 $\pm$ 0.66) & 1774.38 (828.37 $\pm$ 2904.44)\\
0.2 & 0.0005 & 0.30 (0.07 $\pm$ 0.56) & 1948.54 (954.87 $\pm$ 3004.34)\\

    \bottomrule
  \end{tabular}

  \vspace{1ex}
  \footnotesize{  Values are mean (95\% confidence interval). Confidence intervals computed using bootstrapping and shown to two decimal places.}
\end{table}

\newpage

\subsubsection{Ablation: pure \textsc{EXPAND} and \textsc{PROJECT} steps}

In order to isolate the effects of the \textsc{EXPAND} and \textsc{PROJECT} steps, we report the result of running each for $K=3$ iterations in Fig. \ref{fig:projection}. Fig. \ref{fig:projection} shows the Pareto optimal points of the following hyperparameter combinations:

\textsc{EXPAND} step (\AlgNameShortExplore, black circles in Fig. \ref{fig:projection})
\begin{itemize}
    \item $\beta$: $0.0$, $0.1$, $0.2$, $0.3$, $0.4$, $0.5$, $0.6$, $0.7$, $0.8$, $0.9$, $0.95$
    \item $\Tilde{\gamma_0}$: $0.0001$, $0.0002$, $0.0003$, $0.0004$, $0.0005$
\end{itemize}

\textsc{PROJECT} step (green circles in Fig. \ref{fig:projection})
\begin{itemize}
    \item $\eta_k = \eta_0$ constant set to: $0.0$, $0.1$, $0.5$, $1.0$, $2.0$, $5.0$
\end{itemize}

We remark again that running only the \textsc{EXPAND} step without the \textsc{PROJECT} step is equivalent to running \AlgNameShortExplore. For each parameter combination above we average results over 8 seeds.
Notice that the \textsc{EXPAND} and \textsc{PROJECT} steps have the reverse effect of each other: the \textsc{EXPAND} step trades off validity for increased diversity (Conformer Vendi) where as the \textsc{PROJECT} step increases validity at a cost of marginally reduced diversity. 

\begin{figure}[ht]
    \centering
    \includegraphics[width=0.95\linewidth]{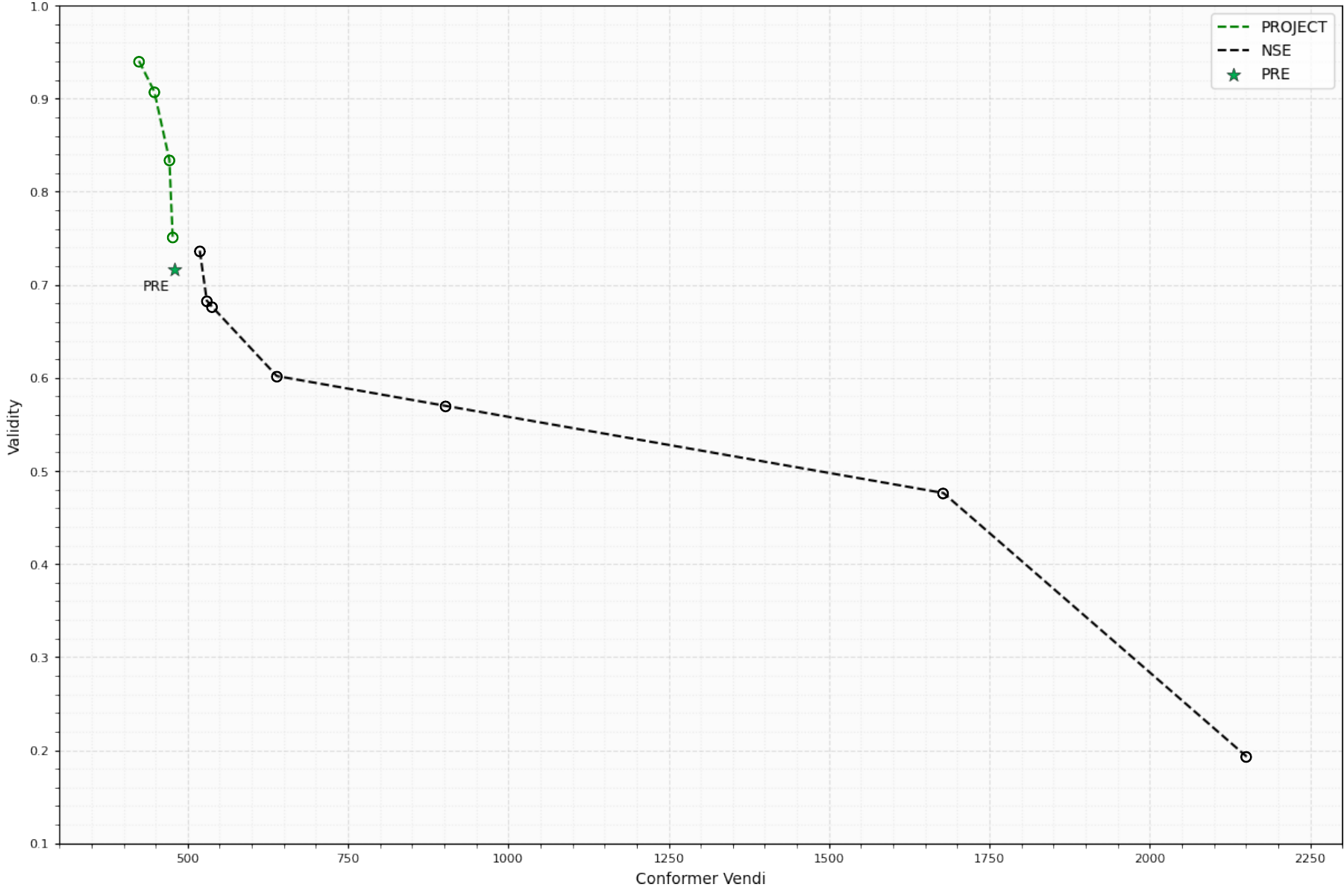}
    \caption{Illustrating the result of isolating the \textsc{EXPAND}  step (\AlgNameShortExplore) and the \textsc{PROJECT} step}
    \label{fig:projection}
\end{figure}

\newpage

\subsubsection{Comparison with Inference Time Filtering}

In this section, we consider an additional comparison of the proposed methods against a baseline corresponding to \AlgNameShortFDC plus inference-time filtering. We report the results in Fig. \ref{fig:inferencefilter} below. In particular, we consider the following algorithms:
\begin{itemize}
    \item \AlgNameShortFDC ~\citep{santi2025flow} + inference-time filtering (blue in Fig. \ref{fig:inferencefilter})
    \item \AlgNameShortExplore (Alg. \ref{alg:algorithm_nse} + inference-time filtering (black in Fig. \ref{fig:inferencefilter})
    \item \AlgNameShortLocal 
\end{itemize}
In particular, for the first two cases above, given a model fine-tuned using \AlgNameShortFDC or \AlgNameShortExplore, we filter the samples generated at inference time by using the same weak verifier employed by \AlgNameShortLocal. We note that this schemes induce a closed-form solution formally mathematically equivalent to conditional sampling, with the condition that the samples are in $\Omega_v$. Thus, this schemes effectively amounts to simulating a perfect projection step, albeit at the price of costly per-sample inference-time filtering (\eg, via rejection sampling). We calculate the Conformer Vendi and validity on each set of 5000 molecules, and report the results in Fig. \ref{fig:inferencefilter} below. As one can notice, \AlgNameShortExplore + filtering shows superior performance compared against \AlgNameShortFDC + filtering, while \AlgNameShortLocal shows slightly less validity than \AlgNameShortExplore + filtering, while being significantly cheaper at sampling time, which is a needed requirement for certain generative modeling applications. Compared with \AlgNameShortFDC + filtering, \AlgNameShortLocal shows superior exploration capabilities and slightly lower validity. Nonetheless, it might be possible to parametrize \AlgNameShortLocal to achieve lower diversity and higher validity, similarly to \AlgNameShortFDC + filtering. In particular, these two algorithms theoretically have the same closed-form solutions, and the concrete differences amount to the high per-sample cost of inference-time filtering, and the superior exploration capabilities of \AlgNameShortLocal and \AlgNameShortExplore over \AlgNameShortFDC likely due to noised space exploration, as discussed in Sec. \ref{alg:algorithm}.

\begin{figure}[ht]
    \centering
    \includegraphics[width=0.95\linewidth]{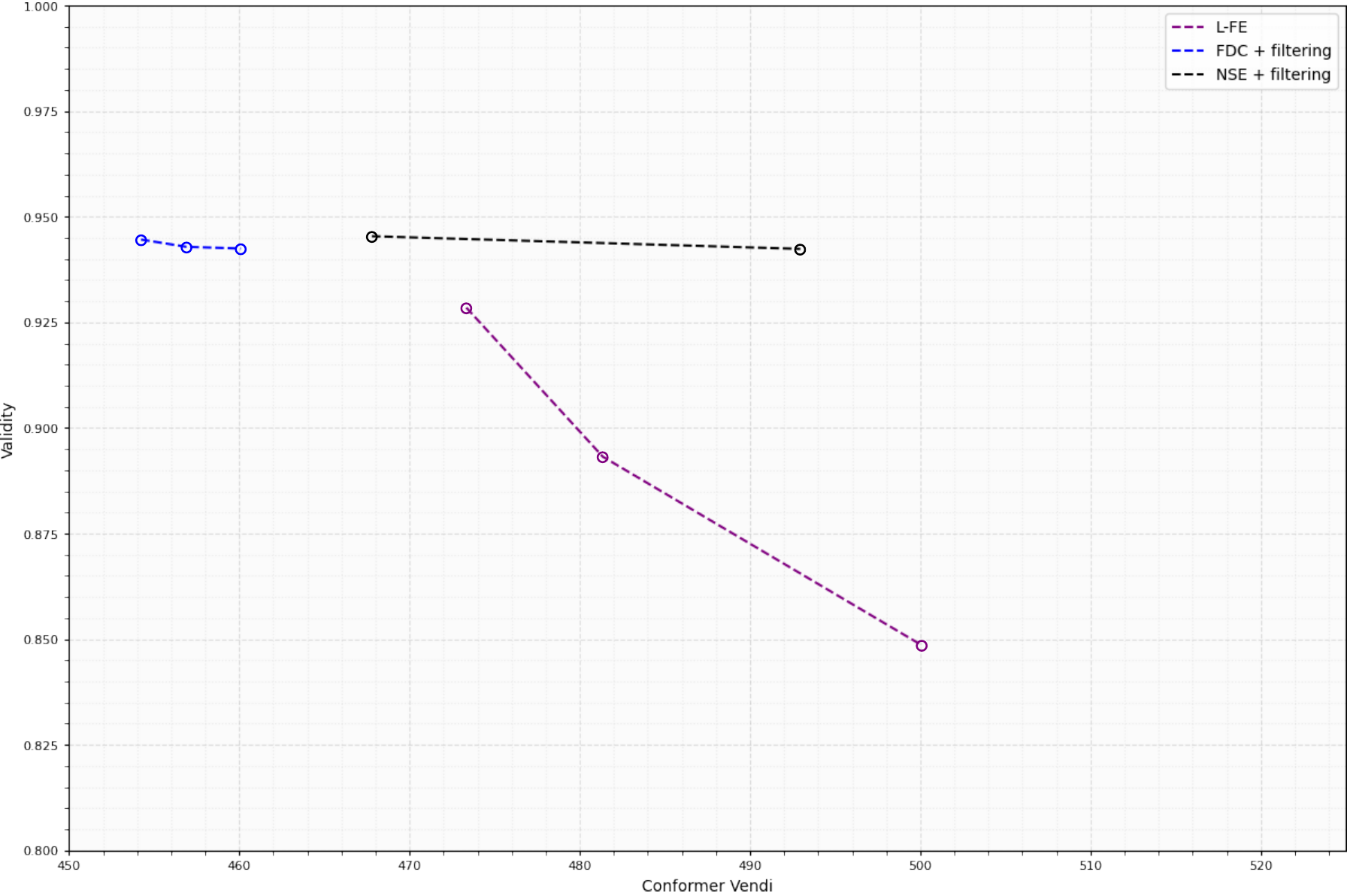}
    \caption{\AlgNameShortLocal compared to \AlgNameShortExplore and \AlgNameShortFDC with inference-time filtering for different parameterizations. We note that \AlgNameShortLocal maintains competitive performance with both methods despite having no filter applied to the output samples, and \AlgNameShortExplore still outperforms \AlgNameShortFDC when both are corrected with an inference-time filter.}
    \label{fig:inferencefilter}
\end{figure}

Each point in Fig. \ref{fig:inferencefilter} corresponds to a different Pareto optimal parametrization ($\beta$, $\gamma$ and $\eta$ if applicable) for each method. Each method is run for $K=3$ iterations, with results averaged across 8 seeds. See section \ref{sec:hypergeom} for details about other hyperparameters.

\subsubsection{Hardware}

We ran all of our experiments using a single NVIDIA RTX 2080Ti GPU per run (QM9 experiments) or a single NVIDIA RTX 4090 GPU per run (GEOM experiments).
\newpage

\section{Update Step Reparametrization}
\label{sec:reparametrization}
Recall the expression for the gradient of running costs in the Local Flow Expander algorithm (Alg. \ref{alg:algorithm}, \eqref{eq:running-cost}):

\begin{align}
    \lambda_t\nabla_{x_t}\delta\mathcal{G}(p^\pi_t) &= \lambda_t\nabla_{x_t}\delta(\mathcal{H}(p^\pi_t) - \alpha_t \mathcal{D}_\textrm{KL}(p^\pi_t || p^\textrm{pre}_t))\\
    &= \lambda_t(-s^\pi_t + \alpha(s^\pi_t - s^\textrm{pre}_t))\\
    &= -\lambda_t((\alpha + 1)s^\pi_t - \alpha s^\textrm{pre}_t)
\end{align}

which are then multiplied by the stepsize $\gamma_k$ at each iteration, resulting in the following expression being plugged into the Adjoint Matching algorithm as the gradient of the running cost:

\begin{equation}
    \nabla f_t = - \gamma_k \lambda_t ((\alpha + 1)s^\pi_t - \alpha s^\textrm{pre}_t) \, .
\end{equation}

While $\alpha$ has an intuitive interpretation as the regularization strength in objective \ref{eq:local_flow_expansion_problem}, it has the unfortunate side-effect of scaling the magnitude of the running cost which could potentially have the opposite effect. Indeed, notice that as $\alpha \rightarrow \infty$ the running costs explode. For practical applications it seems more suitable to reparametrize the running cost as follows:

\begin{equation}
    \nabla f_t = -\Tilde{\gamma}_k \lambda_t (s^\pi_t - \beta s^\textrm{pre}_t)
\end{equation}

for $\beta = \frac{\alpha}{\alpha + 1} \in [0, 1]$, absorbing a $(\alpha + 1)$ factor into the new stepsize $\Tilde{\gamma}_k$:

\begin{equation}
    \Tilde{\gamma}_k = (\alpha + 1) \gamma_k \,.
\end{equation}

Note that this parametrization is as expressive as before but is easier to tune as it disentangles the effect of the stepsize and the regularization strength.

\newpage

\end{document}